\documentclass[english]{article}
\usepackage[T1]{fontenc}
\usepackage[latin9]{inputenc}
\usepackage{amsthm}
\usepackage{amsmath}
\usepackage{amssymb}
\usepackage{graphicx}
\usepackage{esint}
\usepackage[authoryear]{natbib}

\makeatletter

\providecommand{\tabularnewline}{\\}

\theoremstyle{plain}
\newtheorem{thm}{\protect\theoremname}
\theoremstyle{definition}
\newtheorem{defn}[thm]{\protect\definitionname}
\theoremstyle{plain}
\newtheorem{prop}[thm]{\protect\propositionname}
\theoremstyle{plain}
\newtheorem{lem}[thm]{\protect\lemmaname}
\theoremstyle{definition}
\newtheorem{example}[thm]{\protect\examplename}
\ifx\proof\undefined\
\newenvironment{proof}[1][\protect\proofname]{\par
\normalfont\topsep6\p@\@plus6\p@\relax
\trivlist
\itemindent\parindent
\item[\hskip\labelsep
\scshape
#1]\ignorespaces
}{%
\endtrivlist\@endpefalse
}
\fi
\theoremstyle{remark}
\newtheorem{rem}[thm]{\protect\remarkname}
\newcommand{\bb}{\mathbb}

\newcommand{\Cal}{\mathcal}

\usepackage{subfigure}

\usepackage{algorithm}\usepackage{algorithmic}

\usepackage{hyperref}

 \usepackage[accepted]{icml2012_contrib}
\usepackage{natbib,natbibspacing}
\icmltitlerunning{Hypothesis Testing Using Pairwise Distances and Associated
Kernels}

\makeatother

\usepackage{babel}
\providecommand{\definitionname}{Definition}
\providecommand{\examplename}{Example}
\providecommand{\lemmaname}{Lemma}
\providecommand{\proofname}{Proof}
\providecommand{\propositionname}{Proposition}
\providecommand{\remarkname}{Remark}
\providecommand{\theoremname}{Theorem}

\begin{document}
\twocolumn[
\icmltitle{Hypothesis Testing Using Pairwise Distances and Associated Kernels}

\icmlauthor{Dino Sejdinovic$^\star$}{dino.sejdinovic@gmail.com}
\icmlauthor{Arthur Gretton$^{\star,\dagger,\ast}$}{arthur.gretton@gmail.com }
\icmlauthor{Bharath Sriperumbudur$^{\star,\ast}$}{bharath@gatsby.ucl.ac.uk }
\icmlauthor{Kenji Fukumizu$^\ddagger$}{fukumizu@ism.ac.jp }
\icmladdress{$^\star$Gatsby Computational Neuroscience Unit, CSML, University College London, $^\dagger$Max Planck Institute for Intelligent Systems, T\"{u}bingen, $^\ddagger$The Institute of Statistical Mathematics, Tokyo}
\icmlkeywords{hypothesis testing, distance covariance, kernel mean embeddings}

\vskip 0.2in
]
\begin{abstract}
We provide a unifying framework linking two classes of statistics
used in two-sample and independence testing: on the one hand, the
energy distances and distance covariances from the statistics literature;
on the other, distances between embeddings of distributions to reproducing
kernel Hilbert spaces (RKHS), as established in machine learning.
The equivalence holds when energy distances are computed with semimetrics
of negative type, in which case a kernel may be defined such that
the RKHS distance between distributions corresponds exactly to the
energy distance. We determine the class of probability distributions
for which kernels induced by semimetrics are characteristic (that
is, for which embeddings of the distributions to an RKHS are injective).
Finally, we investigate the performance of this family of kernels
in two-sample and independence tests: we show in particular that the
energy distance most commonly employed in statistics is just one member
of a parametric family of kernels, and that other choices from this
family can yield more powerful tests.
\end{abstract}

\section{Introduction}

\label{submission}

The problem of testing statistical hypotheses in high dimensional
spaces is particularly challenging, and has been a recent focus of
considerable work in the statistics and machine learning communities.
On the statistical side, two-sample testing in Euclidean spaces (of
whether two independent samples are from the same distribution, or
from different distributions) can be accomplished using a so-called
energy distance as a statistic \citep{Szekely2004,Szekely2005}. Such
tests are consistent against all alternatives as long as the random
variables have finite first moments. A related dependence measure
between vectors of high dimension is the distance covariance \citep{Szekely2007,SzeRiz09},
and the resulting test is again consistent for variables with bounded
first moment. The distance covariance has had a major impact in the
statistics community, with \citet{SzeRiz09} being accompanied by
an editorial introduction and discussion. A particular advantage of
energy distance-based statistics is their compact representation in
terms of certain expectations of pairwise Euclidean distances, which
leads to straightforward empirical estimates. As a follow-up work,
\citet{Lyons2011} generalized the notion of distance covariance to
metric spaces of negative type (of which Euclidean spaces are a special
case).

On the machine learning side, two-sample tests have been formulated
based on embeddings of probability distributions into reproducing
kernel Hilbert spaces \citep{Gretton2012}, using as the test statistic
the difference between these embeddings: this statistic is called
the maximum mean discrepancy (MMD). This distance measure was applied
to the problem of testing for independence, with the associated test
statistic being the Hilbert-Schmidt Independence Criterion (HSIC)
\citep{GreBouSmoSch05,GreFukTeoSonSchSmo08_short,SmoGreSonSch07,Zhang2011}.
Both tests are shown to be consistent against all alternatives when
a characteristic RKHS is used \citep{FukGreSunSch08_short,SriGreFukLanetal10}.
Such tests can further be generalized to structured and non-Euclidean
domains, such as text strings, graphs or groups \citep{FukSriGreSch09_short}. 

Despite their striking similarity, the link between energy distance-based
tests and kernel-based tests has been an open question. In the discussion
of \citet{SzeRiz09}, \citet[p.~1289]{GreFukSri09} first explored
this link in the context of independence testing, and stated that
interpreting the distance-based independence statistic as a kernel
statistic is not straightforward, since Bochner\textquoteright{}s
theorem does not apply to the choice of weight function used in the
definition of Brownian distance covariance (we briefly review this
argument in Section~\ref{subsec:charac} of the Appendix). \citet[Rejoinder,
p.~1303]{SzeRiz09}
confirmed this conclusion, and commented that RKHS-based dependence
measures do not seem to be formal extensions of Brownian distance
covariance because the weight function is not integrable. Our contribution
resolves this question and shows that RKHS-based dependence measures
are precisely the formal extensions of Brownian distance covariance,
where the problem of non-integrability of weight functions is circumvented
by using translation-variant kernels, i.e., \emph{distance-induced
kernels}, a novel family of kernels that we introduce in Section 2.2.

In the case of two-sample testing, we demonstrate that energy distances
are in fact maximum mean discrepancies arising from the same family
of distance-induced kernels. A number of interesting consequences
arise from this insight: first, we show that the energy distance (and
distance covariance) derives from a particular parameter choice from
a larger family of kernels: this choice may not yield the most sensitive
test. Second, results from \citet{GreFukHarSri09_short,Zhang2011}
may be applied to get consistent two-sample and independence tests
for the energy distance, without using bootstrap, which perform much
better than the upper bound proposed by \citet{Szekely2007} as an
alternative to the bootstrap. Third, in relation to \citet{Lyons2011},
we obtain a new family of characteristic kernels arising from semimetric
spaces of negative type (where the triangle inequality need not hold),
which are quite unlike the characteristic kernels defined via Bochner's
theorem \citep{SriGreFukLanetal10}. 

The structure of the paper is as follows: In Section \ref{sec: Definitions-and-Notation},
we provide the necessary definitions from RKHS theory, and the relation
between RKHS and semimetrics of negative type. In Section \ref{sec:Energy-Distance-and},
we review both the energy distance and distance covariance. We relate
these quantities in Sections \ref{sub:energy_distance_with_kernels}
and \ref{sub:dcov_with_kernels} to the Maximum Mean Discrepancy (MMD)
and the Hilbert-Schmidt Independence Criterion (HSIC), respectively.
We give conditions for these quantities to distinguish between probability
measures in Section \ref{sec:Distinguishing-probability-distributions},
thus obtaining a new family of characteristic kernels. Empirical estimates
of these quantities and associated two-sample and independence tests
are described in Section \ref{sec:Consistency}. Finally, in Section
\ref{sec:Experiments}, we investigate the performance of the test
statistics on a variety of testing problems, which demonstrate the
strengths of the new kernel family.

\section{\label{sec: Definitions-and-Notation}Definitions and Notation}

In this section, we introduce concepts and notation required to understand
reproducing kernel Hilbert spaces (Section \ref{sub:RKHS}), and distribution
embeddings into RKHS. We then introduce semimetrics (Section \ref{sub:Semimetrics-negative-type}),
and review the relation of semimetrics of negative type to RKHS kernels.

\subsection{RKHS Definitions\label{sub:RKHS}}

Unless stated otherwise, we will assume that $\mathcal{Z}$ is any
topological space. \vspace{.5mm}
\begin{defn}
(\textbf{RKHS}) Let $\mathcal{H}$ be a Hilbert space of real-valued
functions defined on $\mathcal{Z}$. A function $k:\mathcal{Z}\times\mathcal{Z}\to\mathbb{R}$
is called \emph{a reproducing kernel} of $\mathcal{H}$ if (i) $\forall z\in\mathcal{Z},\;\; k(\cdot,z)\in\mathcal{H}$,
and (ii) $\forall z\in\mathcal{Z},\,\forall f\in\mathcal{H},\;\;\left\langle f,k(\cdot,z)\right\rangle _{\mathcal{H}}=f(z).$
If $\mathcal{H}$ has a reproducing kernel, it is called \emph{a reproducing
kernel Hilbert space} (RKHS).\vspace{-1mm}
\end{defn}
According to the Moore-Aronszajn theorem \citep[p.~19]{BerTho04},
for every symmetric, positive definite function $k:\mathcal{Z}\times\mathcal{Z}\to\mathbb{R}$,
there is an associated RKHS $\mathcal{H}_{k}$ of real-valued functions
on $\mathcal{Z}$ with reproducing kernel $k$. The map $\varphi:\mathcal{Z}\to\mathcal{H}_{k}$,
$\varphi:z\mapsto k(\cdot,z)$ is called the canonical feature map
or the Aronszajn map of $k$. We will say that $k$ is a nondegenerate
kernel if its Aronszajn map is injective.

\subsection{Semimetrics of Negative Type\label{sub:Semimetrics-negative-type}}

We will work with the notion of semimetric of negative type on a non-empty
set $\mathcal{Z}$, where the {}``distance'' function need not satisfy
the triangle inequality. Note that this notion of semimetric is different
to that which arises from the seminorm, where distance between two
distinct points can be zero (also called pseudonorm). 
\begin{defn}
(\textbf{Semimetric}) Let $\mathcal{Z}$ be a non-empty set and let
$\rho:\mathcal{Z}\times\mathcal{Z}\to[0,\infty)$ be a function such
that $\forall z,z'\in\mathcal{Z}$, (i) $\rho(z,z')=0$ if and only
if $z=z'$, and (ii) $\rho(z,z')=\rho(z',z)$. Then $(\mathcal{Z},\rho)$
is said to be a semimetric space and $\rho$ is called a semimetric
on $\mathcal{Z}$. If, in addition, (iii) $\forall z,z',z''\in\mathcal{Z}$,
$\rho(z',z'')\leq\rho(z,z')+\rho(z,z'')$, $(\mathcal{Z},\rho)$ is
said to be a metric space and $\rho$ is called a metric on $\mathcal{Z}$.
\end{defn}

\begin{defn}
(\textbf{Negative type}) The semimetric space $(\mathcal{Z},\rho)$
is said to have negative type if $\forall n\geq2$, $z_{1},\ldots,z_{n}\in\mathcal{Z}$,
and $\alpha_{1},\ldots,\alpha_{n}\in\mathbb{R}$ with $\sum_{i=1}^{n}\alpha_{i}=0$,
\begin{equation}
\sum_{i=1}^{n}\sum_{j=1}^{n}\alpha_{i}\alpha_{j}\rho(z_{i},z_{j})\leq0.\label{eq: CND}
\end{equation}
\end{defn}
Note that in the terminology of \citet{BerChrRes84}, $\rho$
satisfying \eqref{eq: CND} is said to be a \emph{negative definite}
function. The following theorem is a direct consequence of \citet[Proposition
3.2, p.~82]{BerChrRes84}.
\begin{prop}
\label{pro:semimetric_hilbertian_metric}$\rho$ is a semimetric of
negative type if and only if there exists a Hilbert space $\mathcal{H}$
and an injective map $\varphi:\mathcal{Z}\to\mathcal{H}$, such that
\begin{equation}
\rho(z,z')=\left\Vert \varphi(z)-\varphi(z')\right\Vert
_{\mathcal{H}}^{2}\label{eq: rho_via_Hilbert}\vspace{-2.5mm}
\end{equation}
\end{prop}
This shows that $(\mathbb{R}^{d},\left\Vert \cdot-\cdot\right\Vert ^{2})$
is of negative type. From \citet[Corollary 2.10, p.~78]{BerChrRes84},
we have that:\vspace{.5mm}
\begin{prop}
\label{pro:power and log of cnd}If $\rho$ satisfies \eqref{eq: CND},
then so does $\rho^{q}$, for $0<q<1$.\vspace{-1mm}
\end{prop}
Therefore, by taking $q=1/2$, we conclude that all Euclidean spaces
are of negative type. While \citet[p.~9]{Lyons2011} also uses the
result in Proposition \ref{pro:semimetric_hilbertian_metric}, he
studies embeddings to general Hilbert spaces, and the relation with
the theory of reproducing kernel Hilbert spaces is not exploited.
Semimetrics of negative type and symmetric positive definite kernels
are in fact closely related, as summarized in the following Lemma
based on \citet[Lemma 2.1, p.~74]{BerChrRes84}.\vspace{.5mm}
\begin{lem}
\label{lem:kernel-from-semimetric}Let $\mathcal{Z}$ be a nonempty
set, and let $\rho$ be a semimetric on $\mathcal{Z}$. Let $z_{0}\in\mathcal{Z}$,
and denote $k(z,z')=\rho(z,z_{0})+\rho(z',z_{0})-\rho(z,z')$. Then
$k$ is positive definite if and only if $\rho$ satisfies \eqref{eq:
CND}.\vspace{-1mm}
\end{lem}
We call the kernel $k$ defined above the \emph{distance-induced kernel},
and say that it is\emph{ }induced by the semimetric $\rho$. For brevity,
we will drop {}``induced'' hereafter, and say that $k$ is simply
the \emph{distance kernel} (with some abuse of terminology). In addition,
we will typically work with distance kernels scaled by $1/2$. Note
that $k(z_{0},z_{0})=0$, so distance kernels are not strictly positive
definite (equivalently, $k(\cdot,z_{0})=0$). By varying {}``the
point at the center'' $z_{0}$, one obtains a family $\mathcal{K}_{\rho}=\left\{ \frac{1}{2}\left[\rho(z,z_{0})+\rho(z',z_{0})-\rho(z,z')\right]\right\} _{z_{0}\in\mathcal{Z}}$
of distance kernels induced by $\rho$. We may now express \eqref{eq: rho_via_Hilbert}
from Proposition \ref{pro:semimetric_hilbertian_metric} in terms
of the canonical feature map for the RKHS $\mathcal{H}_{k}$ (proof in
Appendix~\ref{subsec:proofs}).
\begin{prop}
\label{pro: properties of Krho}Let $(\mathcal{Z},\rho)$ be a semimetric
space of negative type, and $k\in\mathcal{K}_{\rho}$. Then: 
\begin{enumerate}
\item $k$ is nondegenerate, i.e., the Aronszajn map $z\mapsto k(\cdot,z)$
is injective.
\item $\rho(z,z')=k(z,z)+k(z',z')-2k(z,z')=\left\Vert
k(\cdot,z)-k(\cdot,z')\right\Vert_{\mathcal{H}_{k}}^{2}.$
\end{enumerate}
\end{prop}
Note that Proposition \ref{pro: properties of Krho} implies that
the Aronszajn map $z\mapsto k(\cdot,z)$ is an isometric embedding
of a metric space $(\mathcal{Z},\rho^{1/2})$ into $\mathcal{H}_{k}$,
for every $k\in\mathcal{K}_{\rho}$.

\subsection{Kernels Inducing Semimetrics}

We now further develop the link between semimetrics of negative type
and kernels. Let $k$ be any nondegenerate reproducing kernel on $\mathcal{Z}$
(for example, every strictly positive definite $k$ is nondegenerate).
Then, by Proposition \ref{pro:semimetric_hilbertian_metric}, 
\begin{equation}
\rho(z,z')=k(z,z)+k(z',z')-2k(z,z')\label{eq:rhoFromK}
\end{equation}
defines a valid semimetric $\rho$ of negative type on $\mathcal{Z}$.
We will say that $k$ generates $\rho$. It is clear that every distance
kernel $\tilde{k}\in\mathcal{K}_{\rho}$ also generates $\rho$, and
that $\tilde{k}$ can be expressed as:
\begin{equation}
\tilde{k}(z,z')=k(z,z')+k(z_{0},z_{0})-k(z,z_{0})-k(z',z_{0}),
\label{eq:generating_kernels}
\end{equation}
for some $z_{0}\in\mathcal{Z}$. In addition, $k\in\mathcal{K}_{\rho}$
if and only if $k(z_{0},z_{0})=0$ for some $z_{0}\in\mathcal{Z}$.
Hence, it is clear that any strictly positive definite kernel, e.g.,
the Gaussian kernel $e^{-\sigma\left\Vert z-z'\right\Vert ^{2}}$,
is \emph{not} a distance kernel.
\begin{example}
\label{exa: various_exponents}Let $\mathcal{Z}=\mathbb{R}^{d}$ and
write $\rho_{q}(z,z')=\left\Vert z-z'\right\Vert ^{q}$. By combining
Propositions \ref{pro:semimetric_hilbertian_metric} and \ref{pro:power and log of cnd},
$\rho_{q}$ is a valid semimetric of negative type for $0<q\leq2$.
It is a metric of negative type if $q\leq1$. The corresponding distance
kernel {}``centered at zero'' is given by
\begin{equation}
k_{q}(z,z')=\frac{1}{2}\left(\left\Vert z\right\Vert ^{q}+\left\Vert z'\right\Vert ^{q}-\left\Vert z-z'\right\Vert ^{q}\right).
\end{equation}

\end{example}

\begin{example}
Let $\mathcal{Z}=\mathbb{R}^{d}$, and consider the Gaussian kernel
$k(z,z')=e^{-\sigma\left\Vert z-z'\right\Vert ^{2}}$. The induced
semimetric is $\rho(z,z')=2\left[1-e^{-\sigma\left\Vert z-z'\right\Vert ^{2}}\right]$.
There are many other kernels that generate $\rho$, however; for example,
the distance kernel induced by $\rho$ and {}``centered at zero''
is $\tilde{k}(z,z')=e^{-\sigma\left\Vert z-z'\right\Vert ^{2}}+1-e^{-\sigma\left\Vert z\right\Vert ^{2}}-e^{-\sigma\left\Vert z'\right\Vert ^{2}}$.
\end{example}

\section{Distances and Covariances}

In this section, we begin with a description of the energy distance,
which measures distance between distributions; and distance covariance,
which measures dependence. We then demonstrate that the former is
a special instance of the maximum mean discrepancy (a kernel measure
of distance on distributions), and the latter an instance of the Hilbert-Schmidt
Independence criterion (a kernel dependence measure). We will denote
by $\mathcal{M}(\mathcal{Z})$ the set of all finite signed Borel
measures on $\mathcal{Z}$, and by $\mathcal{M}_{+}^{1}(\mathcal{Z})$
the set of all Borel probability measures on $\mathcal{Z}$.

\subsection{\label{sec:Energy-Distance-and}Energy Distance and Distance Covariance}

\citet{Szekely2004,Szekely2005} use the following measure of statistical
distance between two probability measures $P$ and $Q$ on $\mathbb{R}^{d}$,
termed the \emph{energy distance}:
\setlength{\arraycolsep}{0.0em}
\begin{eqnarray}
D_{E}(P,Q)&{}={}& 2\mathbb{E}_{ZW}\left\Vert Z-W\right\Vert
-\mathbb{E}_{ZZ'}\left\Vert Z-Z'\right\Vert\nonumber\\
&{}{}&\qquad - \mathbb{E}_{WW'}\left\Vert
W-W'\right\Vert,\label{eq: energy_distance}
\end{eqnarray}
where $Z,Z'\overset{i.i.d.}{\sim}P$ and $W,W'\overset{i.i.d.}{\sim}Q$.
This quantity characterizes the equality of distributions, and in
the scalar case, it coincides with twice the Cramer-Von Mises distance.
We may generalize it to a semimetric space of negative type $(\mathcal{Z},\rho)$,
with the expression for this generalized distance covariance $D_{E,\rho}(P,Q)$
being of the same form as \eqref{eq: energy_distance}, with the Euclidean
distance replaced by $\rho$. Note that the negative type of $\rho$
implies the non-negativity of $D_{E,\rho}$. In Section \ref{sub:energy_distance_with_kernels},
we will show that for every $\rho$, $D_{E,\rho}$ is precisely the MMD
associated to a particular kernel $k$ on $\mathcal{Z}$.

Now, let $X$ be a random vector on $\mathbb{R}^{p}$ and $Y$ a random
vector on $\mathbb{R}^{q}$. The distance covariance was introduced
in \citet{Szekely2007,SzeRiz09} to address the problem of testing
and measuring dependence between $X$ and $Y$, in terms of a weighted
$L_{2}$-distance between characteristic functions of the joint distribution
of $X$ and $Y$ and the product of their marginals. Given a
particular choice of weight function, it can be computed in
terms of certain expectations of pairwise Euclidean distances,
\setlength{\arraycolsep}{0.0em}
\begin{eqnarray}
\label{eq: dCov_in_terms_of_distances}
\mathcal{V}^{2}(X,Y)&{} ={} & \mathbb{E}_{XY}\mathbb{E}_{X'Y'}\left\Vert
X-X'\right\Vert \left\Vert Y-Y'\right\Vert\\
 &{}{}& +  \mathbb{E}_{X}\mathbb{E}_{X'}\left\Vert X-X'\right\Vert
\mathbb{E}_{Y}\mathbb{E}_{Y'}\left\Vert Y-Y'\right\Vert \nonumber \\
 &{}{}&\, -2\mathbb{E}_{X'Y'}\left[\mathbb{E}_{X}\left\Vert
X-X'\right\Vert \mathbb{E}_{Y}\left\Vert Y-Y'\right\Vert \right],\nonumber
\end{eqnarray}
where $(X,Y)$ and $(X',Y')$ are $\overset{i.i.d.}{\sim}P_{XY}$.
Recently, \citet{Lyons2011} established that the generalization of
the distance covariance is possible to metric spaces of negative type,
with the expression for this generalized distance covariance $\mathcal{V}_{\rho_{\mathcal{X}},\rho_{\mathcal{Y}}}^{2}(X,Y)$
being of the same form as \eqref{eq: dCov_in_terms_of_distances},
with Euclidean distances replaced by metrics of negative type $\rho_{\mathcal{X}}$
and $\rho_{\mathcal{Y}}$ on domains $X$ and $Y$, respectively.
In Section \ref{sub:dcov_with_kernels}, we will show that the generalized
distance covariance of a pair of random variables $X$ and $Y$ is
precisely HSIC associated to a particular kernel $k$ on the product
of domains of $X$ and $Y$.

\subsection{\label{sub:energy_distance_with_kernels}Maximum Mean Discrepancy}

The notion of the feature map in an RKHS (Section \ref{sub:RKHS})
can be extended to kernel embeddings of probability measures \citep{BerTho04,SriGreFukLanetal10}. 
\begin{defn}
(\textbf{Kernel embedding}) Let $k$ be a kernel on $\mathcal{Z}$,
and $P\in\mathcal{M}_{+}^{1}(\mathcal{Z})$. The \emph{kernel embedding}
of $P$ into the RKHS $\mathcal{H}_{k}$ is $\mu_{k}(P)\in\mathcal{H}_{k}$
such that $\mathbb{E}_{Z\sim P}f(Z)=\left\langle f,\mu_{k}(P)\right\rangle _{\mathcal{H}_{k}}$
for all $f\in\mathcal{H}_{k}$.
\end{defn}
Alternatively, the kernel embedding can be defined by the Bochner
expectation $\mu_{k}(P)=\mathbb{E}_{Z\sim P}k(\cdot,Z)$. By the Riesz
representation theorem, a sufficient condition for the existence of
$\mu_{k}(P)$ is that $k$ is Borel-measurable and that $\mathbb{E}_{Z\sim P}k^{1/2}(Z,Z)<\infty$.
If $k$ is a bounded continuous function, this is obviously true for
all $P\in\mathcal{M}_{+}^{1}(\mathcal{Z})$. Kernel embeddings can
be used to induce metrics on the spaces of probability measures, giving
the maximum mean discrepancy (MMD), 
\setlength{\arraycolsep}{0.0em}
\begin{eqnarray}
\gamma_{k}^{2}(P,Q) &{} ={} & \left\Vert \mu_{k}(P)-\mu_{k}(Q)\right\Vert
_{\mathcal{H}_{k}}^{2}\nonumber \\
 &{} ={} & \mathbb{E}_{ZZ'}k(Z,Z')+\mathbb{E}_{WW'}k(W,W')\nonumber\\
&{}{}&\qquad - 2\mathbb{E}_{ZW}k(Z,W),\label{eq: MMD}
\end{eqnarray}
where $Z,Z'\overset{i.i.d.}{\sim}P$ and $W,W'\overset{i.i.d.}{\sim}Q$.
If the restriction of $\mu_{k}$ to some $\mathcal{P}(\mathcal{Z})\subseteq\mathcal{M}_{+}^{1}(\mathcal{Z})$
is well defined and injective, then $k$ is said to be characteristic
to $\mathcal{P}(\mathcal{Z})$, and it is said to be characteristic
(without further qualification) if it is characteristic to $\mathcal{M}_{+}^{1}(\mathcal{Z})$.
When $k$ is characteristic, $\gamma_{k}$ is a metric on $\mathcal{M}_{+}^{1}(\mathcal{Z})$,
i.e., $\gamma_{k}\left(P,Q\right)=0$ iff $P=Q$, $\forall P,Q\in\mathcal{M}_{+}^{1}(\mathcal{Z})$.
Conditions under which kernels are characteristic have been studied
by \citet{SriGreFukLanetal08,FukSriGreSch09_short,SriGreFukLanetal10}.
An alternative interpretation of \eqref{eq: MMD} is as an integral
probability metric \citep{Mueller97}: see \citet{Gretton2012} for
details.

In general, distance kernels are continuous but unbounded functions.
Thus, kernel embeddings are not defined for all Borel probability
measures, and one needs to restrict the attention to a class of Borel
probability measures for which $\mathbb{E}_{Z\sim P}k^{1/2}(Z,Z)<\infty$
when discussing the maximum mean discrepancy. We will assume that
all Borel probability measures considered satisfy a stronger condition that $\mathbb{E}_{Z\sim P}k(Z,Z)<\infty$ (this
reflects a finite first moment condition on random variables considered in distance
covariance tests, and will imply that all quantities appearing in our results are well defined). For more details,
see Section~\ref{subsec:restrict} in the Appendix. As an alternative to requiring
this condition, one may assume that
the underlying semimetric space $(\mathcal{Z},\rho)$ of negative
type is itself bounded, i.e., that $\sup_{z,z'\in\mathcal{Z}}\rho(z,z')<\infty$.

We are now able to describe the relation between the maximum mean
discrepancy and the energy distance. The following theorem is a consequence
of Lemma \ref{lem:kernel-from-semimetric}, and is proved in
Section~\ref{subsec:proofs} of the Appendix.
\begin{thm}
\label{thm: 2sample_dkern}Let $(\mathcal{Z},\rho)$ be a semimetric
space of negative type and let $z_{0}\in\mathcal{Z}$. The distance
kernel $k$ induced by $\rho$ satisfies $\gamma_{k}^{2}(P,Q)=\frac{1}{2}D_{E,\rho}(P,Q)$.
In particular, $\gamma_{k}$ does not depend on the choice of $z_{0}$.
\end{thm}
There is a subtlety to the link between kernels and semimetrics, when
used in computing the distance on probabilities. Consider again the
family of distance kernels $\mathcal{K}_{\rho}$, where the semimetric
$\rho$ is itself generated from $k$ according to \eqref{eq:rhoFromK}.
As we have seen, it may be that $k\notin\mathcal{K}_{\rho}$, however
it is clear that $\gamma_{k}^{2}(P,Q)=\frac{1}{2}D_{E,\rho}(P,Q)$
whenever $k$ generates $\rho$. Thus, all kernels that generate the
same semimetric $\rho$ on $\mathcal{Z}$ give rise to the same metric
$\gamma_{k}$ on (possibly a subset of) $\mathcal{M}_{+}^{1}(\mathcal{Z})$,
and $\gamma_{k}$ is merely an extension of the metric $\rho^{1/2}$ on
the point masses. The kernel-based and distance-based methods are
therefore equivalent, provided that we allow {}``distances'' $\rho$
which may not satisfy the triangle inequality.

\subsection{\label{sub:dcov_with_kernels}The Hilbert-Schmidt Independence Criterion}

Given a pair of jointly observed random variables $(X,Y)$ with values
in $\mathcal{X}\times\mathcal{Y}$, the Hilbert-Schmidt Independence
Criterion (HSIC) is computed as the maximum mean discrepancy between
the joint distribution $P_{XY}$ and the product of its marginals
$P_{X}P_{Y}$. Let $k_{\mathcal{X}}$ and $k_{\mathcal{Y}}$ be kernels
on $\mathcal{X}$ and $\mathcal{Y}$, with respective RKHSs $\mathcal{H}_{k_{\mathcal{X}}}$
and $\mathcal{H}_{k_{\mathcal{Y}}}$. Following \citet[Section
2.3]{SmoGreSonSch07},
we consider the MMD associated to the kernel $k\left(\left(x,y\right),\left(x',y'\right)\right)=k_{\mathcal{X}}(x,x')k_{\mathcal{Y}}(y,y')$
on $\mathcal{X}\times\mathcal{Y}$ with RKHS $\mathcal{H}_{k}$ isometrically
isomorphic to the tensor product $\mathcal{H}_{k_{\mathcal{X}}}\otimes\mathcal{H}_{k_{\mathcal{Y}}}$.
It follows that $\theta:=\gamma_{k}^{2}(P_{XY},P_{X}P_{Y})$ with
\setlength{\arraycolsep}{0.0em}
\begin{eqnarray*}
 \theta&{}={}&
 \Bigl\Vert\mathbb{E}_{XY}\left[k_{\mathcal{X}}(\cdot,X)\otimes
k_{\mathcal{Y}}(\cdot,Y)\right]\nonumber\\
 &{}{}&
\,\,-\mathbb{E}_{X}k_{\mathcal{X}}(\cdot,X)\otimes\mathbb{E}_{Y}k_{\mathcal{Y}}
(\cdot
,Y)\Bigr\Vert_{\mathcal{H}_{k_{\mathcal{X}}}\otimes\mathcal{H}_{k_{\mathcal{Y}}}
}^{2}\nonumber\\
 &{} = {}&
\mathbb{E}_{XY}\mathbb{E}_{X'Y'}k_{\mathcal{X}}(X,X')k_{\mathcal{Y}}(Y,
Y')\nonumber\\
 &{}{}&\,\,+
\mathbb{E}_{X}\mathbb{E}_{X'}k_{\mathcal{X}}(X,X')\mathbb{E}_{Y}\mathbb{E}_{Y'}
k_{\mathcal{Y}}(Y,Y')\nonumber\\
 &{}{}&
\quad-2\mathbb{E}_{X'Y'}\left[\mathbb{E}_{X}k_{\mathcal{X}}(X,X')\mathbb{E}_{Y}
k_ { \mathcal{Y}}(Y,Y')\right],\nonumber
\end{eqnarray*}
where in the last step we used that $\left\langle f\otimes g,f'\otimes
g'\right\rangle
_{\mathcal{H}_{k_{\mathcal{X}}}\otimes\mathcal{H}_{k_{\mathcal{Y}}}}
=\left\langle f,f'\right\rangle _{\mathcal{H}_{k_{\mathcal{X}}}}\left\langle
g,g'\right\rangle _{\mathcal{H}_{k_{\mathcal{X}}}}$.
It can be shown that this quantity is the squared Hilbert-Schmidt
norm of the covariance operator between RKHSs \citep{GreHerSmoBouetal05}.
 The following theorem demonstrates the link between HSIC and the
distance covariance, and is proved in Appendix~\ref{subsec:proofs}.
\begin{thm}
\label{thm: dcov_kern}Let $(\mathcal{X},\rho_{\mathcal{X}})$ and
$(\mathcal{Y},\rho_{\mathcal{Y}})$ be semimetric spaces of negative
type, and $(x_{0},y_{0})\in\mathcal{X}\times\mathcal{Y}$. Define
\begin{align}
 & k\left(\left(x,y\right),\left(x',y'\right)\right)\nonumber \\
 & :=\left[\rho_{\mathcal{X}}(x,x_{0})+\rho_{\mathcal{X}}(x',x_{0})-\rho_{\mathcal{X}}(x,x')\right]\times\nonumber \\
 & \quad\:\:\left[\rho_{\mathcal{Y}}(y,y_{0})+\rho_{\mathcal{Y}}(y',y_{0})-\rho_{\mathcal{Y}}(y,y')\right].\label{eq: tensor_kernel}
\end{align}
Then, $k$ is a positive definite kernel on $\mathcal{X}\times\mathcal{Y}$,
and \textup{$\gamma_{k}^{2}(P_{XY},P_{X}P_{Y})=\mathcal{V}_{\rho_{\mathcal{X},}\rho_{\mathcal{Y}}}^{2}(X,Y)$.}
\end{thm}
We remark that a similar result to Theorem \ref{thm: dcov_kern} is
given by \citet[Proposition 3.16]{Lyons2011}, but without making
use of the RKHS equivalence. Theorem \ref{thm: dcov_kern} is a more
general statement, in the sense that we allow $\rho$ to be a semimetric
of negative type, rather than a metric. In addition to yielding a more
general statement,  the RKHS equivalence leads to a significantly
simpler proof: the result is an immediate application of the HSIC
expansion of \citet{SmoGreSonSch07}.

\section{\label{sec:Distinguishing-probability-distributions}Distinguishing
Probability Distributions}

\citet[Theorem 3.20]{Lyons2011} shows that distance covariance in
a metric space characterizes independence if the metrics satisfy an
additional property, termed \emph{strong negative type}. We will extend
this notion to a semimetric $\rho$. We will say that $P\in\mathcal{M}_{+}^{1}(\mathcal{Z})$
has a finite first moment w.r.t.~$\rho$ if $\int\rho(z,z_{0})dP$
is finite for some $z_{0}\in\mathcal{Z}$. It is easy to see that the integral $\int\rho\,\,
d\left(\left[P-Q\right]\times\left[P-Q\right]\right)=-D_{E,\rho}(P,Q)$ converges whenever $P$ and $Q$ have finite first moments w.r.t. $\rho$. In Appendix~\ref{subsec:restrict}, we
show that this condition is equivalent to $\mathbb{E}_{Z\sim
P}k(Z,Z)<\infty$,
for a kernel $k$ that generates $\rho$, which implies  the kernel
embedding $\mu_{k}(P)$ is also well defined.
\begin{defn}
The semimetric space $(\mathcal{Z},\rho)$ is said to have \emph{a
strong negative type} if $\forall P,Q\in\mathcal{M}_{+}^{1}(\mathcal{Z})$
with finite first moment w.r.t.~$\rho$, {\small 
\begin{equation}
P\neq Q\Rightarrow\int\rho\,\,
d\left(\left[P-Q\right]\times\left[P-Q\right]\right)<0.\label{eq: CISND}
\end{equation}
}{\small \par}
\end{defn}
The quantity in \eqref{eq: CISND} is exactly $-2\gamma_{k}^{2}(P,Q)$
for all $P,Q$ with finite first moment w.r.t.~$\rho$. We directly
obtain:
\begin{prop}
\label{pro:strong_negative_char}Let kernel $k$ generate $\rho$.
Then $(\mathcal{Z},\rho)$ has a strong negative type if and only
if $k$ is characteristic to all probability measures with finite
first moment w.r.t.~\textup{ $\rho$}.
\end{prop}
Thus, the problems of checking whether a semimetric is of strong negative
type and whether its associated kernel is characteristic to an appropriate
space of Borel probability measures are equivalent. This conclusion
has some overlap with \citet{Lyons2011}: in particular, Proposition
\ref{pro:strong_negative_char} is stated in \citet[Proposition 3.10]{Lyons2011},
where the barycenter map $\beta$ is a kernel embedding in our terminology,
although Lyons does not consider distribution embeddings in an RKHS.

\section{\label{sec:Consistency}Empirical Estimates and Hypothesis Tests}

In the case of two-sample testing, we are given i.i.d. samples
$\mathbf{z}=\left\{ z_{i}\right\} _{i=1}^{m}\sim P$
and $\mathbf{w}=\left\{ w_{i}\right\} _{i=1}^{n}\sim Q$. The empirical
(biased) V-statistic estimate of \eqref{eq: MMD} is
\setlength{\arraycolsep}{0.0em}
\begin{eqnarray}
\hat{\gamma}_{k,V}^{2}(\mathbf{z},\mathbf{w}) &{} ={} &
\frac{1}{m^{2}}\sum_{i=1}^{m}\sum_{j=1}^{m}k(z_{i},z_{j})+\frac{1}{n^{2}}\sum_{
i=1}^{n}\sum_{j=1}^{n}k(w_{i},w_{j})\nonumber \\
 &{}{}&\qquad -\frac{2}{mn}\sum_{i=1}^{m}\sum_{j=1}^{n}k(z_{i},w_{j}).
\label{eq: empirical_mmd}
\end{eqnarray}
Recall that if we use a distance kernel $k$ induced by a semimetric
$\rho$, this estimate involves only the pairwise $\rho$-distances
between the sample points.

In the case of independence testing, we are given i.i.d. samples $\mathbf{z}=\left\{ (x_{i},y_{i})\right\} _{i=1}^{m}\sim P_{XY}$,
and the resulting V-statistic estimate (HSIC) is \citep{GreBouSmoSch05,GreFukTeoSonSchSmo08_short}
\begin{equation}
HSIC(\mathbf{z};k_{\mathcal{X}},k_{\mathcal{Y}})=\frac{1}{m^{2}}Tr(K_{\mathcal{X}}HK_{\mathcal{Y}}H),\label{eq: empirical_HSIC}
\end{equation}
where $K_{\mathcal{X}}$, $K_{\mathcal{Y}}$ and $H$ are $m\times m$
matrices given by $\left(K_{\mathcal{X}}\right)_{ij}:=k_{\mathcal{X}}(x_{i},x_{j})$,
$\left(K_{\mathcal{Y}}\right)_{ij}:=k_{\mathcal{Y}}(y_{i},y_{j})$
and $H_{ij}=\delta_{ij}-\frac{1}{m}$ (centering matrix). As in the
two-sample case, if both $k_{\mathcal{X}}$ and $k_{\mathcal{Y}}$
are distance kernels, the test statistic involves only the pairwise
distances between the samples, i.e., kernel matrices in \eqref{eq: empirical_HSIC}
may be replaced by distance matrices. 

We would like to design distance-based tests with an asymptotic Type
I error of $\alpha$, and thus we require an estimate of the $\left(1-\alpha\right)$-quantile
of the V-statistic distribution under the null hypothesis. Under the
null hypothesis, both \eqref{eq: empirical_mmd} and \eqref{eq: empirical_HSIC}
converge to a particular weighted sum of chi-squared distributed independent
random variables (for more details, see Section~\ref{subsec:tests}). We
investigate two approaches, both
of which yield consistent tests: a bootstrap approach \citep{ArcGin92},
and a spectral approach \citep{GreFukHarSri09_short,Zhang2011}. The
latter requires empirical computation of the spectrum of kernel integral
operators, a problem studied extensively in the context of kernel
PCA \citep{SchSmoMul97}. In the two-sample case, one computes the
eigenvalues of the centred Gram matrix $\tilde{K}=HKH$ on the aggregated
samples. Here, $K$ is a $2m\times2m$ matrix, with entries $K_{ij}=k(u_{i},u_{j})$,
$\mathbf{u}=[\mathbf{z}\;\mathbf{w}]$ is the concatenation of the
two samples and $H$ is the centering matrix. \citet{GreFukHarSri09_short}
show that the null distribution defined using these finite sample estimates
converges to the population distribution, provided that the spectrum is square-root summable. The
same approach can be used for a consistent finite sample
null distribution of HSIC, via computation of the eigenvalues
of $\tilde{K}_{\mathcal{X}}=HK_{\mathcal{X}}H$ and $\tilde{K}_{\mathcal{Y}}=HK_{\mathcal{Y}}H$
\citep{Zhang2011}.

Both \citet[p.~14]{Szekely2004} and \citet[p.~2782--2783]{Szekely2007}
establish that the energy distance and distance covariance statistics,
respectively, converge to a particular weighted sum of chi-squares
of form similar to that found for the kernel-based statistics. Analogous
results for the generalized distance covariance are presented by
\citet[p.~7--8]{Lyons2011}.
These works do not propose test designs that attempt to estimate
the coefficients in such representations of the null distribution, however (note also
that these coefficients have a more intuitive interpretation using
kernels). Besides the bootstrap, \citet[Theorem 6]{Szekely2007} also
proposes an independence test using a bound applicable to a general
quadratic form $Q$ of centered Gaussian random variables with $\mathbb{E}[Q]=1$:
$\mathbb{P}\left\{ Q\geq\left(\Phi^{-1}(1-\alpha/2)^{2}\right)\right\} \leq\alpha$,
valid for $0<\alpha\leq0.215$. When applied to the distance covariance
statistic, the upper bound of $\alpha$ is achieved if $X$ and $Y$
are independent Bernoulli variables. The authors remark that the resulting
criterion might be over-conservative. Thus, more sensitive tests are
possible by computing the spectrum of the centred Gram matrices associated
to distance kernels, and we pursue this approach in the next section.

\section{\label{sec:Experiments}Experiments}

\subsection{Two-sample Experiments}

In the two-sample experiments, we investigate three different kinds
of synthetic data. In the first, we compare two multivariate Gaussians,
where the means differ in one dimension only, and all variances are
equal. In the second, we again compare two multivariate Gaussians,
but this time with identical means in all dimensions, and variance
that differs in a single dimension. In our third experiment, we use
the benchmark data of \citet{SriFukGreLanetal09_short}: one distribution
is a univariate Gaussian, and the second is a univariate Gaussian
with a sinusoidal perturbation of increasing frequency (where higher
frequencies correspond to harder problems). All tests use a distance
kernel induced by the Euclidean distance. As shown on the left plots
in Figure~\ref{fig:Gaussian-vs-Dist},  the spectral and
bootstrap test designs appear indistinguishable, and they significantly
outperform the test designed using the quadratic form bound, which
appears to be far too conservative for the data sets considered. This
is confirmed by checking the Type I error of the quadratic form test,
which is significantly smaller than the test size of $\alpha=0.05$.

We also compare the performance to that of the Gaussian kernel, with
the bandwidth set to the median distance between points in the aggregation
of samples. We see that when the means differ, both tests perform
similarly. When the variances differ, it is clear that the Gaussian
kernel has a major advantage over the distance kernel, although this
advantage decreases with increasing dimension (where both perform
poorly). In the case of a sinusoidal perturbation, the performance
is again very similar.

\begin{figure}[t]
\begin{centering}
\begin{tabular}{cc}
\includegraphics[bb=88bp 255bp 485bp
590bp,clip,width=0.232\textwidth]{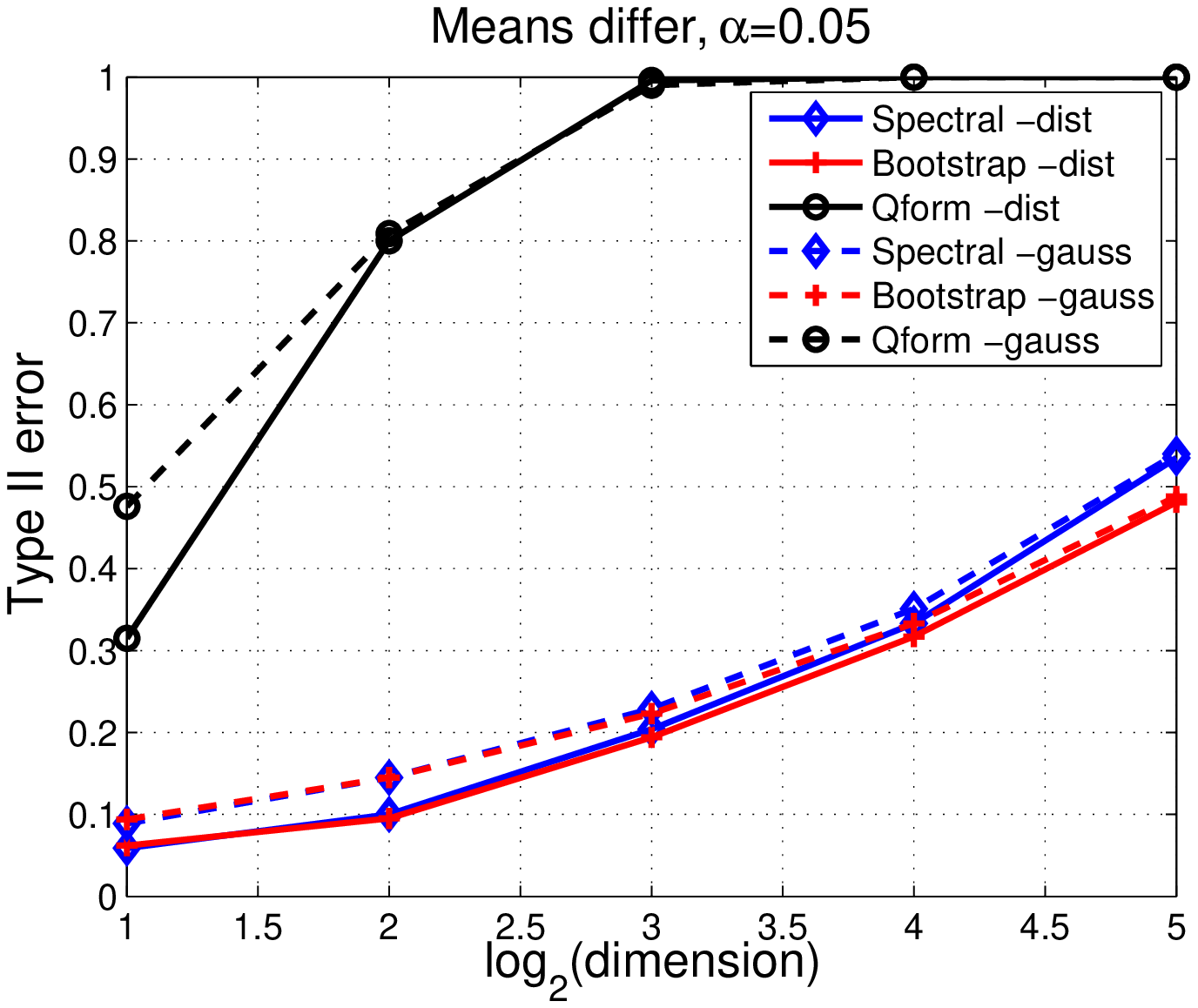}
\includegraphics[bb=88bp 255bp 485bp
590bp,clip,width=0.232\textwidth]{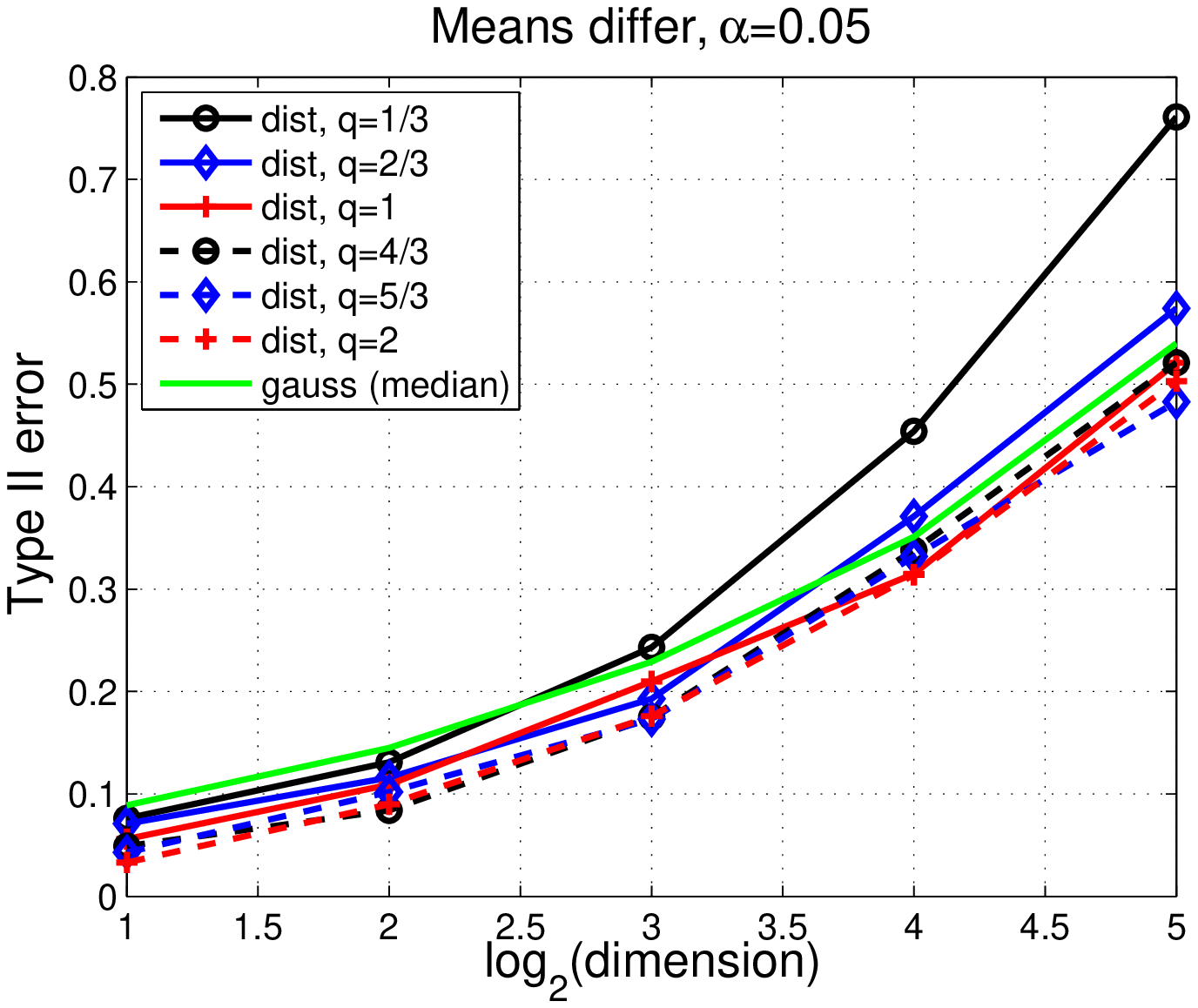}
\end{tabular}
\end{centering}

\begin{centering}
\begin{tabular}{cc}
\includegraphics[bb=88bp 255bp 485bp 590bp,
clip,width=0.232\textwidth]{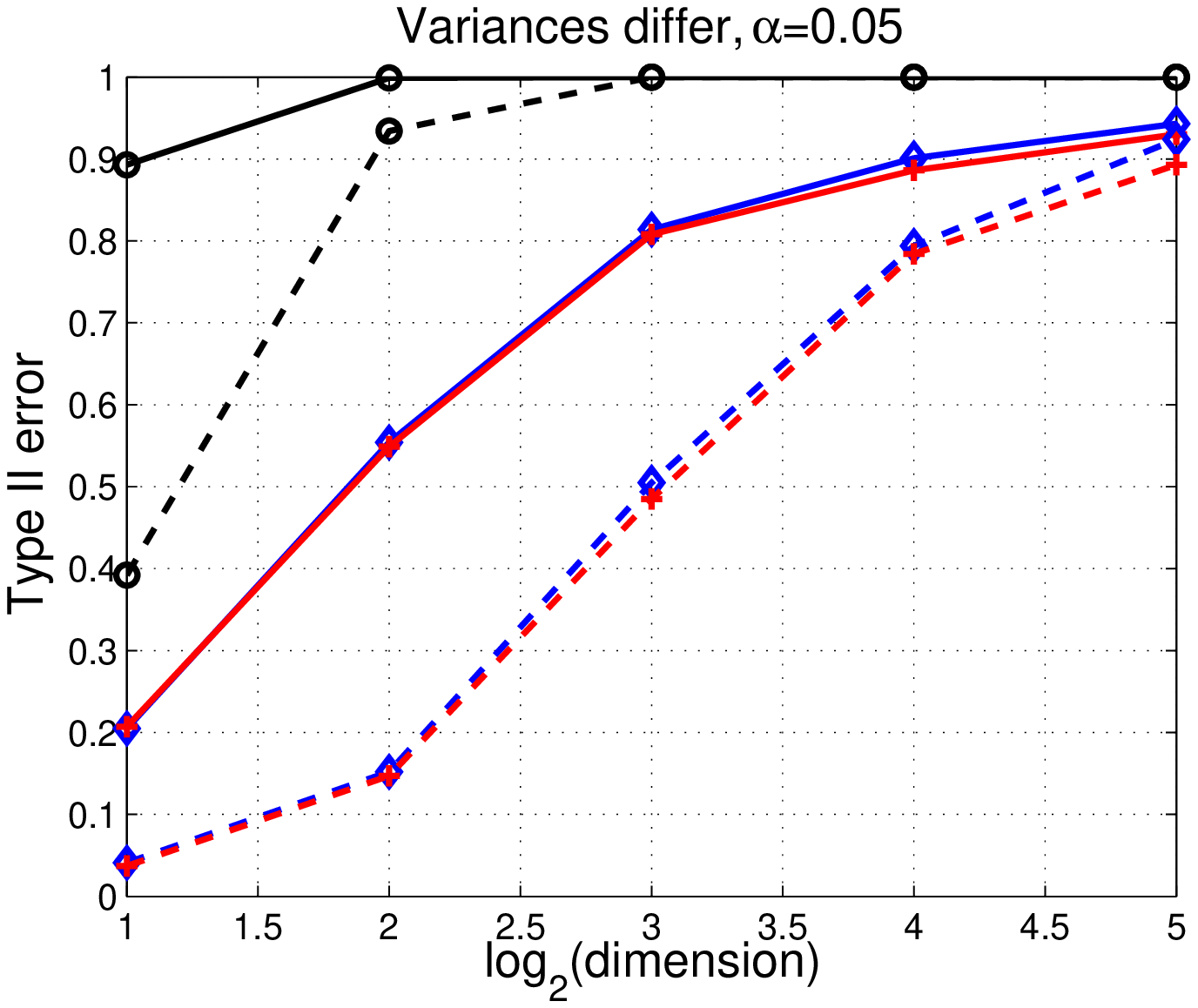}
\includegraphics[bb=88bp 255bp 485bp
590bp,clip,width=0.232\textwidth]{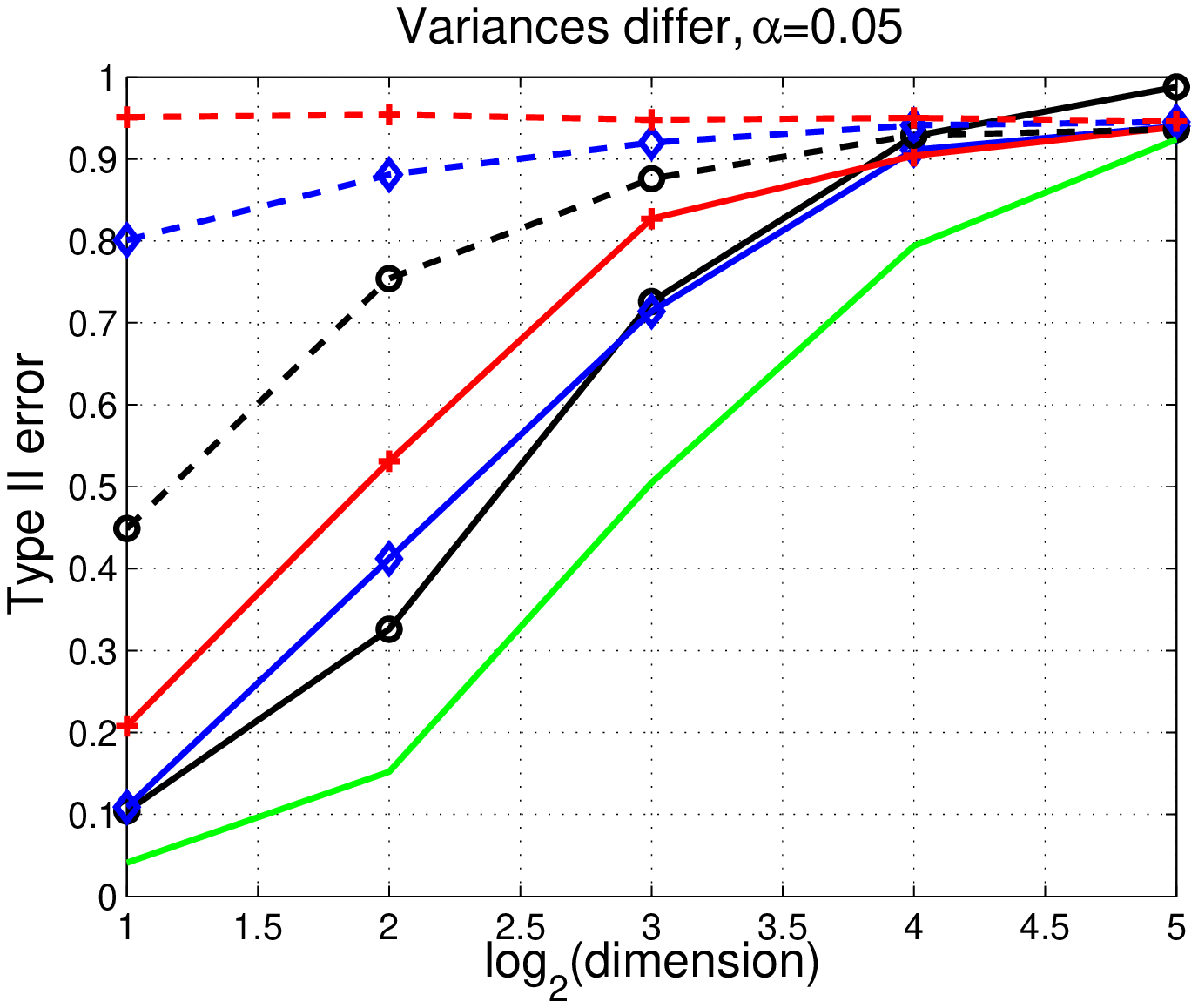}
\end{tabular}
\end{centering}

\begin{centering}
\begin{tabular}{cc}
\includegraphics[bb=88bp 255bp 485bp 590bp,
clip,width=0.232\textwidth]{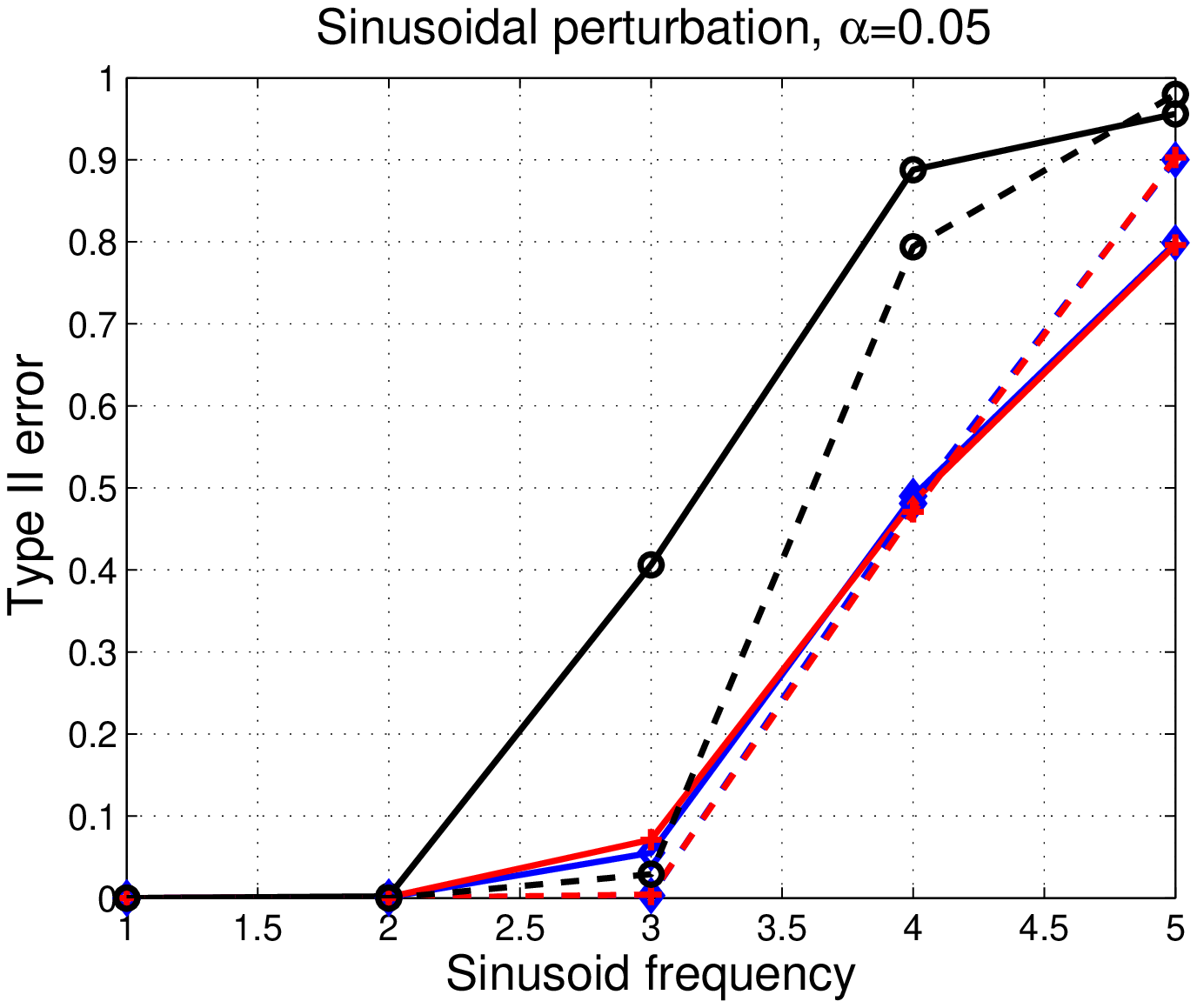}
\includegraphics[bb=88bp 255bp 485bp
590bp,clip,width=0.232\textwidth]{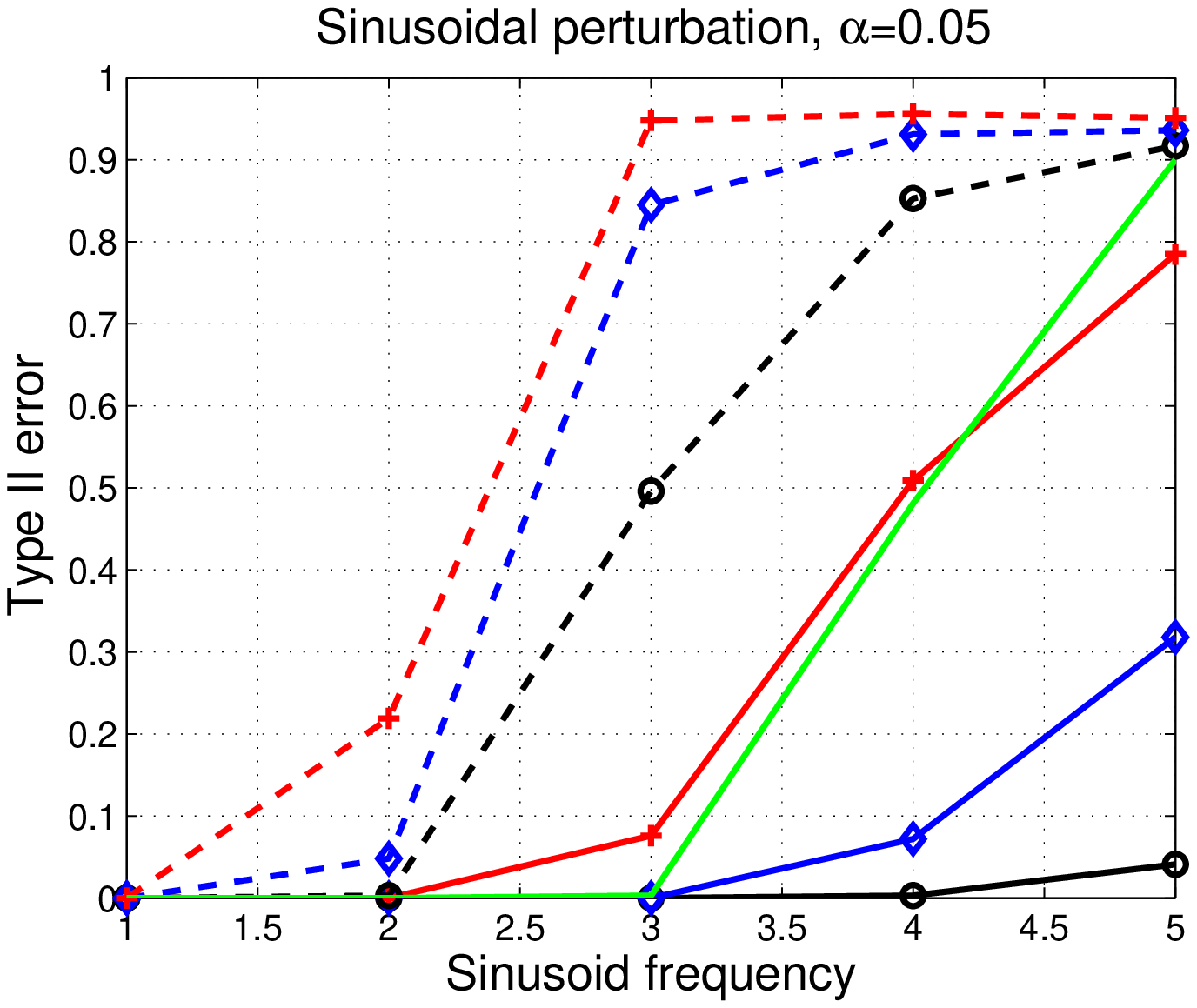}
\end{tabular}
\vspace{-2mm}
\end{centering}
\caption{\label{fig:Gaussian-vs-Dist}(left) MMD using Gaussian and distance
kernels for various tests; (right) Spectral MMD using distance kernels
with various exponents. The number of samples in all experiments was
set to $m=200$.}\vspace{-6mm}
\end{figure}
In addition, following Example~\ref{exa: various_exponents}, we investigate the
performance of kernels obtained using the semimetric $\rho(z,z')=\left\Vert
z-z'\right\Vert ^{q}$
for $0<q\leq2$. Results are presented in the right hand plots
of Figure~\ref{fig:Gaussian-vs-Dist}. While  judiciously
chosen values of $q$ offer some improvement in the cases of differing
mean and variance, we see a dramatic improvement 
for the sinusoidal perturbation,
compared with the case $q=1$
and the Gaussian kernel:
 values $q=1/3$ (and smaller) yield virtually error-free
performance even at high frequencies (note that $q=1$ corresponds to the
energy distance described in \citet{Szekely2004,Szekely2005}). 
Additional experiments with real-world data are presented in
Appendix~\ref{subsec:further}. 

We observe from the simulation results that distance kernels with
higher exponents are advantageous in cases where distributions differ
in mean value along a single dimension (with noise in the remainder),
whereas distance kernels with smaller exponents are more sensitive
to differences in distributions at finer lengthscales (i.e., where
the characteristic functions of the distributions differ at higher
frequencies). This observation also appears to hold true on the real-world
data experiments in Appendix~\ref{subsec:further}.

\subsection{Independence Experiments}

\begin{figure}[t]
\begin{centering}
\begin{tabular}{cc}
\includegraphics[bb=88bp 255bp 485bp 590bp,
clip,width=0.232\textwidth]{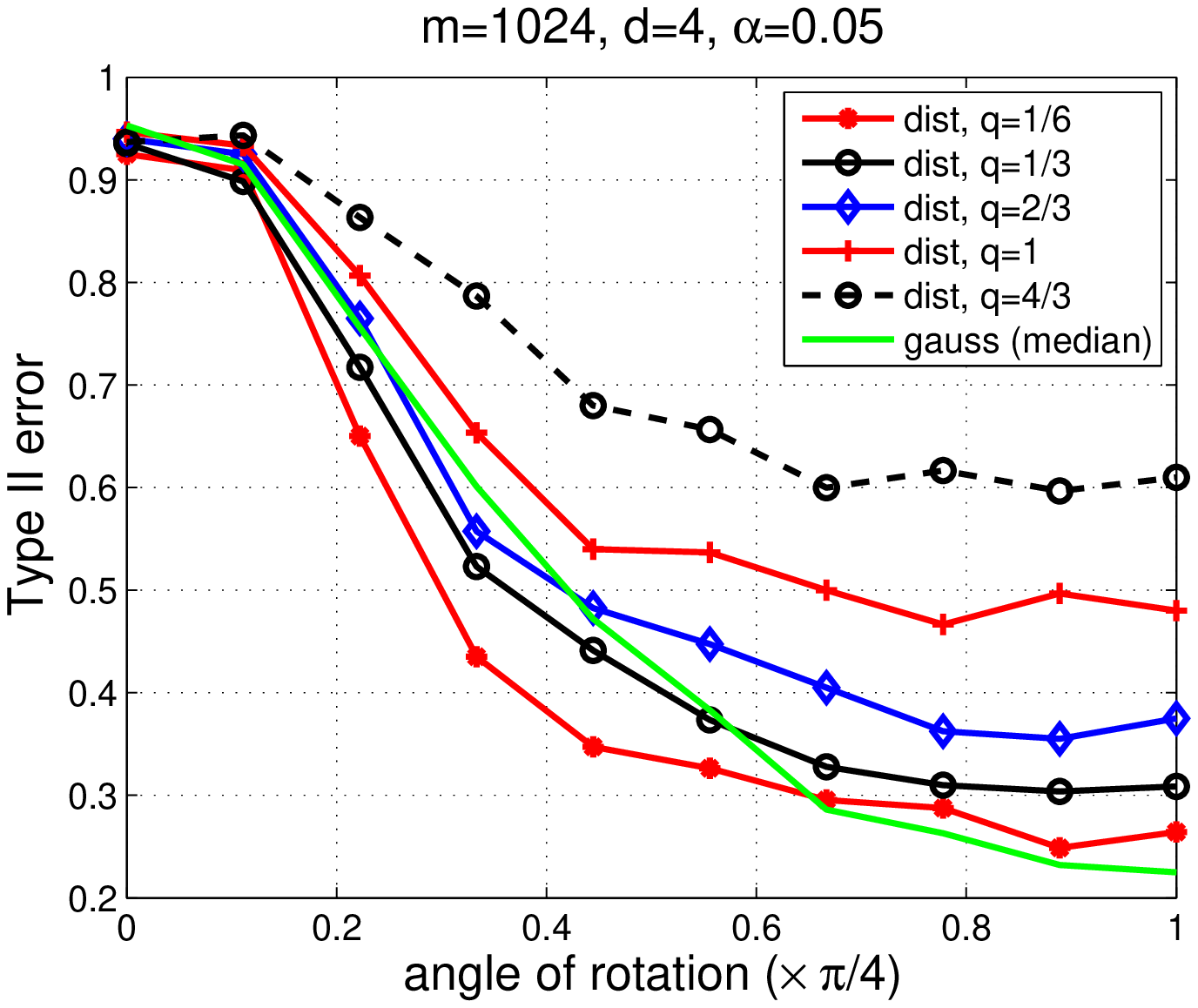}
\includegraphics[bb=88bp 255bp 485bp
590bp,clip,width=0.232\textwidth]{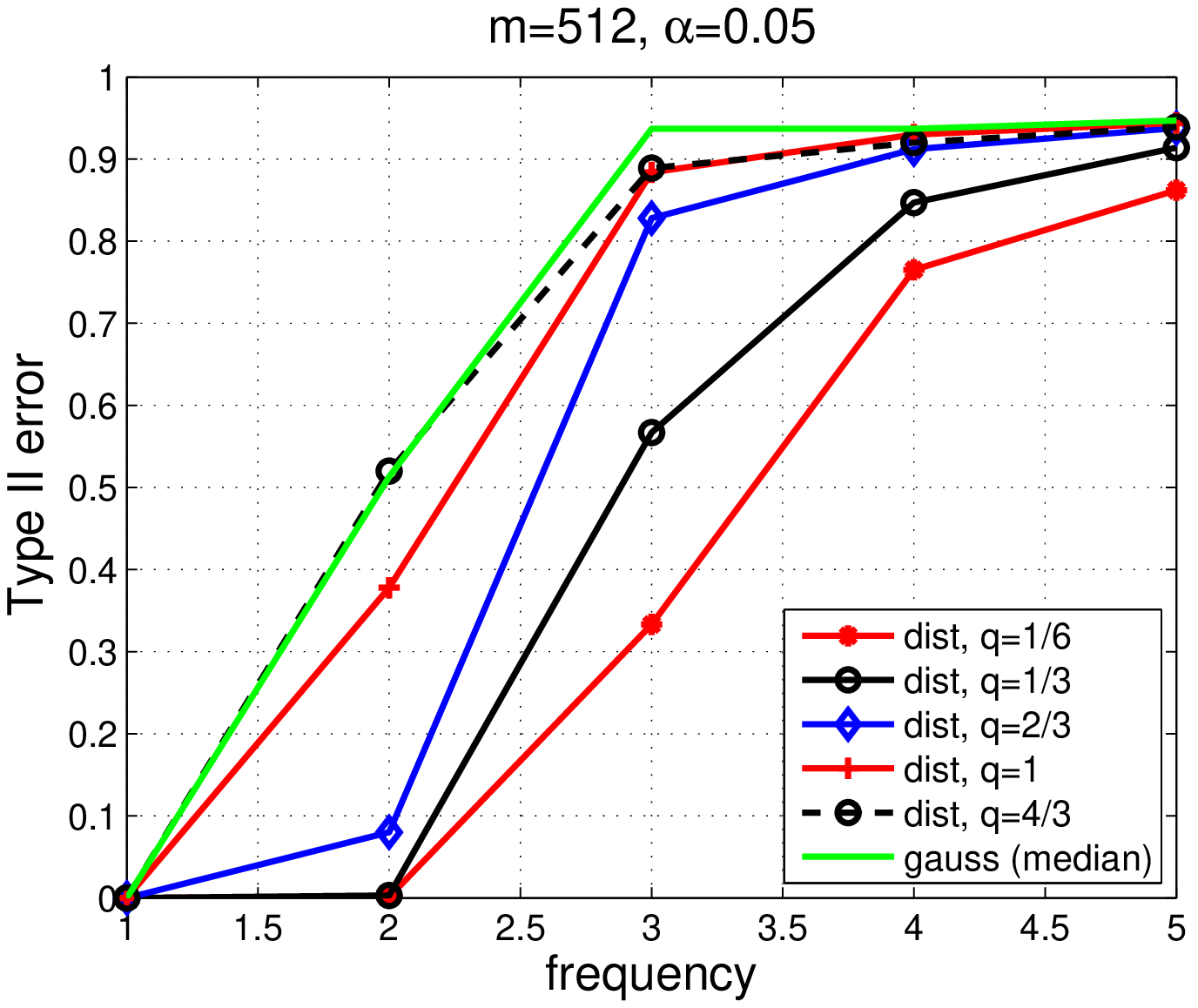}
\end{tabular}\vspace{-2mm}
\end{centering}
\caption{\label{fig:qdist_HSIC_ALL}HSIC using distance kernels with various
exponents and a Gaussian kernel as a function of (left) the angle
of rotation for the dependence induced by rotation; (right) frequency
$\ell$ in the sinusoidal dependence example.}\vspace{-6mm}
\end{figure}
To assess independence tests, we used an artificial benchmark proposed
by \citet{GreFukTeoSonSchSmo08_short}: we generate univariate random
variables from the ICA benchmark densities of \citet{BacJor02}; rotate
them in the product space by an angle between $0$ and $\pi/4$ to
introduce dependence; fill additional dimensions with independent
Gaussian noise; and, finally, pass the resulting multivariate data
through random and independent orthogonal transformations. The resulting
random variables $X$ and $Y$ are dependent but uncorrelated. The
case $m=1024$ (sample size) and $d=4$ (dimension) is plotted in
Figure~\ref{fig:qdist_HSIC_ALL} (left). As observed by \citet{GreFukSri09},
the Gaussian kernel does better than the distance kernel with $q=1$.
By varying $q$,  however, we are able to obtain a wide range of performance;
in particular, the values $q=1/6$ (and smaller) have an advantage
over the Gaussian kernel on this dataset, especially in the case
of  smaller angles of rotation. As for the two-sample case, bootstrap
and spectral tests have indistinguishable performance, and are significantly
more sensitive than the quadratic form based test, which failed to
reject the null hypothesis of independence on this dataset. 

In addition, we assess the test performance on sinusoidally dependent data.
The distribution over the random variable pair $X,Y$ was drawn from
$P_{XY}\propto1+\sin(\ell x)\sin(\ell y)$ for integer $\ell$, on
the support $\mathcal{X}\times\mathcal{Y}$, where $\mathcal{X}:=[-\pi,\pi]$
and $\mathcal{Y}:=[-\pi,\pi]$. In this way, increasing $\ell$ caused
the departure from a uniform (independent) distribution to occur at
increasing frequencies, making this departure harder to detect from
a small sample size. Results are in Figure~\ref{fig:qdist_HSIC_ALL}
(right). We note that the distance covariance outperforms the Gaussian
kernel on this example, and that smaller exponents result in better
performance (lower Type II error when the departure from independence
occurs at higher frequencies). Finally, we note that the setting $q=1$,
which is described in \citet{Szekely2007,SzeRiz09}, is a reasonable
heuristic in practice, but does not yield the most powerful tests
on either dataset.

\section{Conclusion}

We have established an equivalence between the energy distance and
distance covariance, and RKHS measures of distance between distributions.
In particular, energy distances and RKHS distance measures coincide
when the kernel is induced by a semimetric of negative type. The associated
family of kernels performs well in two-sample and independence testing:
interestingly, the parameter choice most commonly used in the statistics
literature does not yield the most powerful tests in many settings.

The interpretation of the energy distance and distance covariance
in an RKHS setting should be of considerable interest both to statisticians
and machine learning researchers, since the associated kernels may
be used much more widely: in conditional dependence testing and estimates
of the chi-squared distance \citep{FukGreSunSch08_short}, in Bayesian
inference \citep{FukSonGre11_Bayes}, in mixture density estimation
\citep{Sriperumbudur11} and
in other machine learning
applications. In particular, the link with kernels makes these applications
of the energy distance immediate and straightforward. Finally, for
problem settings  defined most naturally in terms of distances,
and where these distances are of negative type, there is an interpretation
in terms of reproducing kernels, and the learning machinery from the
kernel literature can be brought to bear.

\footnotesize
\bibliographystyle{icml2012_short}
\bibliography{example_paper,local,bibfile}
\normalsize
\appendix
\normalsize

\section{\label{sec:Homogeneity-Testing-in-1}Appendix}

\subsection{Proofs}\label{subsec:proofs}
\begin{proof}
(\textbf{Proposition} \ref{pro: properties of Krho}) If $z,z'\in\mathcal{Z}$
are such that $k(w,z)=k(w,z')$, for all $w\in\mathcal{Z}$, one would
also have $\rho(z,z_{0})-\rho(z,w)=\rho(z',z_{0})-\rho(z',w)$, for
all $w\in\mathcal{Z}$. In particular, by inserting $w=z$, and $w=z'$,
we obtain $\rho(z,z')=-\rho(z,z')=0$, i.e., $z=z'$. The second statement
follows readily by expressing $k$ in terms of $\rho$. \vspace{-4mm}
\end{proof}

\begin{proof}
(\textbf{Theorem} \ref{thm: 2sample_dkern}) Follows directly by inserting
the distance kernel from Lemma \ref{lem:kernel-from-semimetric} into
\eqref{eq: MMD}, and cancelling out the terms dependant on a single
random variable. Define $\theta:=\gamma_{k}^{2}(P,Q)$.
\setlength{\arraycolsep}{0.0em}
\begin{eqnarray*}
\theta &{} ={} &
\frac{1}{2}\mathbb{E}_{ZZ'}\left[\rho(Z,z_{0})+\rho(Z',z_{0})-\rho(Z,Z')\right]
\\
 &{}{}&\quad+\frac{1}{2}\mathbb{E}_{WW'}\left[\rho(W,z_{0})+\rho(W',z_{0}
)-\rho(W ,W')\right] \\
 &{} {}&\qquad
-\mathbb{E}_{ZW}\left[\rho(Z,z_{0})+\rho(W,z_{0})-\rho(Z,W)\right]\\
 &{} ={} &
\mathbb{E}_{ZW}\rho(Z,W)-\frac{\mathbb{E}_{ZZ'}\rho(Z,Z')}{2}-\frac{\mathbb{E}_
{WW'}\rho(W, W')}{2}.
\end{eqnarray*}
\vspace{-3mm}
\end{proof}

\begin{proof}
(\textbf{Theorem} \ref{thm: dcov_kern}) First, we note that $k$
is a valid reproducing kernel since
$k\left(\left(x,y\right),\left(x',y'\right)\right)=k_{\mathcal{X}}(x,x')k_{
\mathcal{Y}}(y,y')$,
where we have taken
$k_{\mathcal{X}}(x,x')=\rho_{\mathcal{X}}(x,x_{0})+\rho_{\mathcal{X}}(x',x_{0}
)-\rho_{\mathcal{X}}(x,x')$,
and
$k_{\mathcal{Y}}(y,y')=\rho_{\mathcal{Y}}(y,y_{0})+\rho_{\mathcal{Y}}(y',y_{0}
)-\rho_{\mathcal{Y}}(y,y')$,
as distance kernels induced by $\rho_{\mathcal{X}}$ and $\rho_{\mathcal{Y}}$,
respectively. Indeed, a product of two reproducing kernels is always
a valid reproducing kernel on the product space \citep[Lemma 4.6,
p.~114]{Steinwart2008book}.
To show equality to distance covariance, we start by expanding
$\theta:=\gamma_{k}^{2}(P_{XY},P_{X}P_{Y})$,
\setlength{\arraycolsep}{0.0em}
\begin{eqnarray*}
 \theta&{}={}&
 \overbrace{\mathbb{E}_{XY}\mathbb{E}_{X'Y'}k_{\mathcal{X}}(X,X')k_{\mathcal{Y}}
(Y , Y')}^{\theta_1}\\
 &{}{}&\quad + 
\overbrace{\mathbb{E}_{X}\mathbb{E}_{X'}k_{\mathcal{X}}(X,X')\mathbb{E}_{Y}
\mathbb { E } _ { Y' } k_{\mathcal{Y}}(Y,Y')}^{\theta_2}\\
 &{}{}&\qquad
-2\overbrace{\mathbb{E}_{X'Y'}\left[\mathbb{E}_{X}k_{\mathcal{X}}(X,X')\mathbb{E
} _ { Y } k_ { \mathcal{Y}}(Y,Y')\right]}^{\theta_3}.
\end{eqnarray*}
Note that
\begin{eqnarray*}
\lefteqn{\theta_1=\mathbb{E}_{XY}\mathbb{E}_{X'Y'}\rho_{\mathcal{X}}(X,X')\rho_{
\mathcal{Y}}(Y, Y')}\\
 && + 
2\mathbb{E}_{X}\rho_{\mathcal{X}}(X,x_{0})\mathbb{E}_{Y}\rho_{\mathcal{Y}}(Y,y_{
0})\\
&& +  2\mathbb{E}_{XY}\rho_{\mathcal{X}}(X,x_{0})\rho_{\mathcal{Y}}(Y,y_{0})\\
 && - 
2\mathbb{E}_{XY}\left[\rho_{\mathcal{X}}(X,x_{0})\mathbb{E}_{Y'}\rho_{\mathcal{Y
}}(Y,Y')\right]\\
 && -
2\mathbb{E}_{XY}\left[\rho_{\mathcal{Y}}(Y,y_{0})\mathbb{E}_{X'}\rho_{\mathcal{X
}}(X,X')\right],
\end{eqnarray*}
\begin{eqnarray*}
\lefteqn{\theta_2=
\mathbb{E}_{X}\mathbb{E}_{X'}\rho_{\mathcal{X}}(X,X')\mathbb{E}_{Y}\mathbb{E}_{
Y'}\rho_{\mathcal{Y}}(Y,Y')}\\
 && + 
4\mathbb{E}_{X}\rho_{\mathcal{X}}(X,x_{0})\mathbb{E}_{Y}\rho_{\mathcal{Y}}(Y,y_{
0})\\
 && - 
2\mathbb{E}_{X}\rho_{\mathcal{X}}(X,x_{0})\mathbb{E}_{Y}\mathbb{E}_{Y'}\rho_{
\mathcal{Y}}(Y,Y')\\
 && -
2\mathbb{E}_{Y}\rho_{\mathcal{Y}}(Y,y_{0})\mathbb{E}_{X}\mathbb{E}_{X'}\rho_{
\mathcal{X}}(X,X'),
\end{eqnarray*}
and
\begin{eqnarray*}
 \lefteqn{\theta_3=
\mathbb{E}_{X'Y'}\left[\mathbb{E}_{X}\rho_{\mathcal{X}}(X,X')\mathbb{E}_{Y}\rho_
{\mathcal{Y}}(Y,Y')\right]}\\
 && + 
3\mathbb{E}_{X}\rho_{\mathcal{X}}(X,x_{0})\mathbb{E}_{Y}\rho_{\mathcal{Y}}(Y,y_{
0})\\
 && +  \mathbb{E}_{XY}\rho_{\mathcal{X}}(X,x_{0})\rho_{\mathcal{Y}}(Y,y_{0})\\
 && - 
\mathbb{E}_{XY}\left[\rho_{\mathcal{X}}(X,x_{0})\mathbb{E}_{Y'}\rho_{\mathcal{Y}
}(Y,Y')\right]\\
 && -
\mathbb{E}_{XY}\left[\rho_{\mathcal{Y}}(Y,y_{0})\mathbb{E}_{X'}\rho_{\mathcal{X}
}(X,X')\right]\\
 && - 
\mathbb{E}_{X}\rho_{\mathcal{X}}(X,x_{0})\mathbb{E}_{Y}\mathbb{E}_{Y'}\rho_{
\mathcal{Y}}(Y,Y')\\
 && - 
\mathbb{E}_{Y}\rho_{\mathcal{Y}}(Y,y_{0})\mathbb{E}_{X}\mathbb{E}_{X'}\rho_{
\mathcal{X}}(X,X').
\end{eqnarray*}
The claim now follows by inserting the resulting expansions and cancelling
the appropriate terms. Note that only the leading terms in the expansions
remain.\end{proof}
\begin{rem}
It turns out that $k$ is not characteristic to
$\mathcal{M}_{+}^{1}(\mathcal{X}\times\mathcal{Y})$
--- i.e., it cannot distinguish between any two distributions on
$\mathcal{X}\times\mathcal{Y}$,
even if $k_{\mathcal{X}}$ and $k_{\mathcal{Y}}$ are characteristic.
However, since $\gamma_{k}$ is equal to the Brownian distance covariance,
we know that it can always distinguish between any $P_{XY}$ and its
product of marginals $P_{X}P_{Y}$ in the Euclidean case. Namely,
note that $k((x_{0},y),(x_{0},y'))=k((x,y_{0}),(x',y_{0}))=0$ for
all $x,x'\in\mathcal{X}$, $y,y'\in\mathcal{Y}$. That means that
for every two distinct $P_{Y},Q_{Y}\in\mathcal{M}_{+}^{1}(\mathcal{Y})$,
one has $\gamma_{k}^{2}(\delta_{x_{0}}P_{Y},\delta_{x_{0}}Q_{Y})=0$.
Thus, kernel in \eqref{eq: tensor_kernel} characterizes independence
but not equality of probability measures on the product space. Informally
speaking, the independence testing is an easier problem than homogeneity
testing on the product space.
\end{rem}

\subsection{Spectral Tests}\label{subsec:tests}

Assume that the null hypothesis holds, i.e., that $P=Q$. For a kernel $k$ and
a Borel probability measure $P$, define a kernel ``centred''
at $P$:
$\tilde{k}_{P}(z,z'):=k(z,z')+\mathbb{E}_{WW'}k(W,W')-\mathbb{E}_{W}k(z,
W)-\mathbb{E}_{W}k(z',W)$,
with $W,W'\overset{i.i.d.}{\sim}P$. Note that as a special case
for $P=\delta_{z_{0}}$ we recover the family of kernels in
\eqref{eq:generating_kernels},
and that $\mathbb{E}_{ZZ'}\tilde{k}_{P}(Z,Z')=0$, i.e.,
$\mu_{\tilde{k}_{P}}(P)=0$.
The centred kernel is important in characterizing the null distribution
of the V-statistic. To the centred kernel $\tilde{k}_{P}$ on domain
$\mathcal{Z}$, one associates the \emph{integral kernel operator}
$S_{\tilde{k}_{P}}:L_{P}^{2}(\mathcal{Z})\to L_{P}^{2}(\mathcal{Z})$
\citep[see][p.~126--127]{Steinwart2008book}, given by:
\begin{eqnarray}
S_{\tilde{k}_{P}}g(z) & = &
\int_{\mathcal{Z}}\tilde{k}_{P}(z,w)g(w)\,dP(w).\label{eq: kernel_operator}
\end{eqnarray}
The following theorem is a special case of \citet[Theorem 12]{Gretton2012}.
For simplicity, we focus on the case where $m=n$.
\begin{thm}
\label{thm: null_2sample}Let $\mathbf{\mathbf{Z}}=\left\{ Z_{i}\right\}
_{i=1}^{m}$
and $\mathbf{W}=\left\{ W_{i}\right\} _{i=1}^{m}$ be two i.i.d. samples
from $P\in\mathcal{M}_{+}^{1}(\mathcal{Z})$, and let $S_{\tilde{k}_{P}}$ be a
trace class operator.
Then
\begin{eqnarray}
\frac{m}{2}\hat{\gamma}_{k,V}^{2}(\mathbf{Z},\mathbf{W}) & \rightsquigarrow &
\sum_{i=1}^{\infty}\lambda_{i}N_{i}^{2},\label{eq: null_dist}
\end{eqnarray}
where $N_{i}\overset{i.i.d.}{\sim}\mathcal{N}(0,1)$, $i\in\mathbb{N}$,
and \textup{$\left\{ \lambda_{i}\right\} _{i=1}^{\infty}$} are the
eigenvalues of the operator $S_{\tilde{k}_{P}}$.
\end{thm}
Note that this result requires that the integral kernel operator associated to
the underlying probability measure $P$ is a trace class operator, i.e., that $\mathbb{E}_{Z\sim P}k(Z,Z)<\infty$. 
As before, the sufficient condition for this to hold for all probability measures
is that $k$ is a bounded function. In the case of a distance kernel,
this is the case if the domain $\mathcal{Z}$ has a bounded diameter
with respect to the semimetric $\rho$, i.e., that
$\sup_{z,z'\in\mathcal{Z}}\rho(z,z')<\infty$.

The null distribution of HSIC takes an analogous form to \eqref{eq: null_dist}
of a weighted sum of chi-squares, but with coefficients corresponding
to the products of the eigenvalues of integral operators $S_{\tilde{k}_{P_{X}}}$
and $S_{\tilde{k}_{P_{Y}}}$. The following Theorem is in \citet[Theorem
4]{Zhang2011}
and gives an asymptotic form for the null distribution of HSIC. See
also \citet[Remark 2.9]{Lyons2011}.
\begin{thm}
\label{thm: null_indep-1}Let $\mathbf{\mathbf{Z}}=\left\{
\left(X_{i},Y_{i}\right)\right\} _{i=1}^{m}$
be an i.i.d. sample from $P_{XY}=P_{X}P_{Y}$, with values in
$\mathcal{X}\times\mathcal{Y}$.
Let $S_{\tilde{k}_{P_{X}}}:L_{P_{X}}^{2}(\mathcal{X})\to
L_{P_{X}}^{2}(\mathcal{X})$,
and $S_{\tilde{k}_{P_{Y}}}:L_{P_{Y}}^{2}(\mathcal{Y})\to
L_{P_{Y}}^{2}(\mathcal{Y})$ be trace class operators.
Then
\begin{eqnarray}
mHSIC(\mathbf{Z};k_{\mathcal{X}},k_{\mathcal{Y}}) & \rightsquigarrow &
\sum_{i=1}^{\infty}\sum_{j=1}^{\infty}\lambda_{i}\eta_{j}N_{i,j}^{2},\label{eq:
null_dist-hsic-1}
\end{eqnarray}
where $N_{i,j}\sim\mathcal{N}(0,1)$, $i,j\in\mathbb{N}$, are independent
and \textup{$\left\{ \lambda_{i}\right\} _{i=1}^{\infty}$} and $\left\{
\eta_{j}\right\} _{j=1}^{\infty}$
are the eigenvalues of the operators $S_{\tilde{k}_{P_{X}}}$
and $S_{\tilde{k}_{P_{Y}}}$,
respectively.
\end{thm}
Note that if $\mathcal{X}$ and $\mathcal{Y}$ have bounded diameters
w.r.t.~$\rho_{\mathcal{X}}$ and $\rho_{\mathcal{Y}}$, Theorem \ref{thm:
null_indep-1}
applies to distance kernels induced by $\rho_{\mathcal{X}}$ and
$\rho_{\mathcal{Y}}$
for all $P_{X}\in\mathcal{\ensuremath{M}}_{+}^{1}(\mathcal{X})$,
$P_{Y}\in\mathcal{\ensuremath{M}}_{+}^{1}(\mathcal{Y})$ .

\subsection{A Characteristic Function Based Interpretation}\label{subsec:charac}

The distance covariance in \eqref{eq: dCov_in_terms_of_distances}
was defined by \citet{Szekely2007} in terms of a weighted distance
between characteristic functions. We briefly review this interpretation
here, however we show that this approach \emph{cannot} be used to derive
a kernel-based measure of dependence (this result was first noted
by \citet{GreFukSri09}, and is included here in the interests of
completeness). Let $X$ be a random vector on $\mathcal{X=}\mathbb{R}^{p}$
and $Y$ a random vector on $\mathcal{Y}=\mathbb{R}^{q}$. The characteristic
function of $X$ and $Y$, respectively, will be denoted by $f_{X}$
and $f_{Y}$, and their joint characteristic function by $f_{XY}$.
The distance covariance $\mathcal{V}(X,Y)$ is defined via the norm
of $f_{XY}-f_{X}f_{Y}$ in a weighted $L_{2}$ space on $\mathbb{R}^{p+q}$,
i.e., 
\begin{equation}
\mathcal{V}^{2}(X,Y)=\int\left|f_{X,Y}(t,s)-f_{X}(t)f_{Y}
(s)\right|^{2}w(t,s)\,dt\,ds,\label{eq: dcov_via_characteristic-1}
\end{equation}
for a particular choice of weight function given by 
\begin{equation}
w(t,s)=\frac{1}{c_{p}c_{q}}\cdot\frac{1}{\left\Vert t\right\Vert
^{1+p}\left\Vert s\right\Vert ^{1+q}},\label{eq: dcov_weight-1}
\end{equation}
where $c_{d}=\pi^{\frac{1+d}{2}}/\Gamma(\frac{1+d}{2})$, $d\geq1$.
An important aspect of distance covariance is that $\mathcal{V}(X,Y)=0$
if and only if $X$ and $Y$ are independent. We next obtain a similar
statistic in the kernel setting. Write
$\mathcal{Z}=\mathcal{X}\times\mathcal{Y}$,
and let $k(z,z')=\kappa(z-z')$ be a translation invariant RKHS kernel
on $\mathcal{Z}$, where $\kappa:\mathcal{Z}\to\mathbb{R}$ is a bounded
continuous function. Using Bochner's theorem, $\kappa$ can be written
as:
$$\kappa(z)=\int e^{-z^{\top}u}d\Lambda(u),$$

for a finite non-negative Borel measure $\Lambda$. It follows
\citet{GreFukSri09}
that
\[
\gamma_{k}^{2}(P_{XY},P_{X}P_{Y})=\int\left|f_{X,Y}(t,s)-f_{X
}(t)f_{Y}(s)\right|^{2}d\Lambda(t,s),
\]
which is in clear correspondence with \eqref{eq: dcov_via_characteristic-1}.
However, the weight function in \eqref{eq: dcov_weight-1} is not
integrable --- so one cannot find a translation invariant kernel for
which $\gamma_{k}$ coincides with the distance covariance. By contrast,
note the kernel in \eqref{eq: tensor_kernel} is \emph{not} translation
invariant.

\subsection{Restriction on Probability Measures}\label{subsec:restrict}

In general, distance kernels and their products are continuous but
unbounded, so kernel embeddings are not defined for all Borel probability
measures. Thus, one needs to restrict the attention to a particular
class of Borel probability measures for which kernel embeddings exist, and a sufficient condition for this is that $\mathbb{E}_{Z\sim P}k^{1/2}(Z,Z)<\infty$, by the Riesz representation theorem. Let $k$ be
a measurable reproducing kernel on $\mathcal{Z}$, and denote, for $\theta>0$,

{\footnotesize 
\begin{equation}
\label{mk}
\mathcal{M}_{k}^{\theta}(\mathcal{Z})=\left\{ \nu\in\mathcal{M}(\mathcal{Z})\,:\,\int
k^{\theta}(z,z)\,d\left|\nu\right|(z)<\infty\right\} .
\end{equation}
}
Note that the maximum mean discrepancy $\gamma_k(P,Q)$ is well defined $\forall
P,Q\in\mathcal{M}_{k}^{1/2}(\mathcal{Z})\cap\mathcal{M}_{+}^{1}(\mathcal{Z})$.

Now, let $\rho$ be a semimetric of negative type. Then, we can consider the class of probability measures that have a finite $\theta$-moment with respect to $\rho$:
\begin{eqnarray}
\label{mrho}
\mathcal{M}_{\rho}^{\theta}(\mathcal{Z})=\{ \nu\in\mathcal{M}(\mathcal{Z})\,&{}:{}&\,\exists z_0\in\mathcal Z,\\ {}& s.t.{}\; & \int
\rho^{\theta}(z,z_0)\,d\left|\nu\right|(z)<\infty\}\nonumber.
\end{eqnarray}
To ensure existence of energy distance $D_{E,\rho}(P,Q)$, we need to assume that $P,Q\in \mathcal{M}_{\theta}^{1}(\mathcal{Z})$, as otherwise expectations $\mathbb E_{ZZ'}\rho(Z,Z')$, $\mathbb E_{WW'}\rho(W,W')$ and $\mathbb E_{ZW}\rho(Z,W)$ may be undefined.
The following proposition shows that the classes of probability measures in \eqref{mk} and \eqref{mrho} coincide at $\theta=n/2$, for $n\in\mathbb N$, whenever $\rho$ is generated by kernel $k$.  
\begin{prop}
Let $k$ be a kernel that generates semimetric $\rho$, and let $n\in\mathbb N$. Then, $\mathcal{M}_{k}^{n/2}(\mathcal{Z})=\mathcal{M}_{\theta}^{n/2}(\mathcal{Z})$.
In particular, if $k_{1}$ and $k_{2}$ generate the same semimetric
$\rho$, then
\textup{$\mathcal{M}^{n/2}_{k_{1}}(\mathcal{Z})=\mathcal{M}^{n/2}_{k_{2}}(\mathcal{Z})$.}
\end{prop}
\begin{proof}
Let $\theta\ge\frac{1}{2}$. Note that $a^{2\theta}$ is a convex
function of $a$. Suppose $\nu\in \mathcal{M}^\theta_k(\mathcal{Z})$. Then, we have
\small
\begin{eqnarray*}
{}&{}&\int \rho^\theta(z,z_0)\,d|\nu|(z)\\
{}&{}&=\int \Vert
k(\cdot,z)-k(\cdot,x_0)\Vert^{2\theta}_{\mathcal{H}_k}\,d|\nu|(z)\\ 
{}&{}&\le \int \left(\Vert k(\cdot,z)\Vert_{\mathcal{H}_k} + \Vert
k(\cdot,z_0)\Vert_{\mathcal{H}_k}\right)^{2\theta}\,d|\nu|(z)\nonumber\\
{}&{}&\le 2^{2\theta-1}\left(\int \Vert k(\cdot,z)\Vert^{2\theta}_{\mathcal{H}_k}\,d|\nu|(z)
+ \int \Vert
k(\cdot,z_0)\Vert^{2\theta}_{\mathcal{H}_k}\,d|\nu|(z)\right)\nonumber\\
{}&{}&=2^{2\theta-1}\left(\int
k^\theta(z,z)\,d|\nu|(z)+k^\theta(z_0,z_0)|\nu|(\Cal{Z})\right)\nonumber\\
{}&{}&<\;\infty,\nonumber 
\end{eqnarray*}
\normalsize
where we have invoked the Jensen's inequality for convex functions. From the
above it is clear that
$\Cal{M}^\theta_k(\Cal{Z})\subset\Cal{M}^\theta_\rho(\Cal{Z})$, for $\theta\ge 1/2$.
\par To prove the other direction, we show by induction that
$\Cal{M}^\theta_\rho(\Cal{Z})\subset
\Cal{M}^{n/2}_k(\Cal{Z})$ for $\theta\ge \frac{n}{2}$, $n\in\bb{N}$.
Let $n=1$. Let $\theta\ge \frac{1}{2}$, and suppose that
$\nu\in\Cal{M}^\theta_\rho(\Cal{X})$. Then, by invoking the reverse triangle and Jensen's inequalities, we have:
\begin{eqnarray*}
\lefteqn{\int\rho^{\theta}(z,z_0)d\left|\nu\right|(z)
 = \int\left\Vert k(\cdot,z)-k(\cdot,z_{0})\right\Vert^{2\theta}
_{\mathcal{H}_{k}}d|\nu|(z)}\\
 && \qquad\quad\geq
\int\left|k^{1/2}(z,z)-k^{1/2}(z_{0},z_{0})\right|^{2\theta}d|\nu|(z)\\
 && \qquad\quad\geq \Big|\int k^{1/2}(z,z)\,d|\nu|(z)-\left\Vert \nu\right\Vert
_{TV}k^{1/2}(z_{0},z_{0})\Big|^{2\theta},
\end{eqnarray*}
which implies $\nu\in\Cal{M}^{1/2}_k(\Cal{Z})$, thereby satisfying the
result for $n=1$. Suppose the result holds for $\theta\ge
\frac{n-1}{2}$, i.e., $\Cal{M}^\theta_\rho(\Cal{Z})\subset
\Cal{M}^{(n-1)/2}_k(\Cal{Z})$ for $\theta\ge\frac{n-1}{2}$. Let
$\nu\in \Cal{M}^\theta_\rho(\Cal{Z})$ for $\theta\ge \frac{n}{2}$. Then we have
\small
\begin{eqnarray*}
{}&{}&\int\rho^\theta(z,z_0)\,d|\nu|(z)\\
{}&{}&=\;\int \left(\Vert
k(\cdot,z)-k(\cdot,z_0)\Vert^{n}_{\mathcal{H}_k}\right)^\frac{2\theta}{n}\,
d|\nu|(z)\\
{}&{}&\ge\;\left(\int \Vert
k(\cdot,z)-k(\cdot,z_0)\Vert^{n}_{\mathcal{H}_k}\,
d|\nu|(z)\right)^\frac{2\theta}{n}\\
{}&{}&\ge\;\left(\int \left|\Vert
k(\cdot,z)\Vert_{\mathcal{H}_k}-\Vert k(\cdot,z_0)\Vert_{\mathcal{H}_k}\right|^{n}\,
d|\nu|(z)\right)^\frac{2\theta}{n}\\
{}&{}&\ge\;\left|\int \left(\Vert
k(\cdot,z)\Vert_{\mathcal{H}_k}-\Vert k(\cdot,z_0)\Vert_{\mathcal{H}_k}\right)^{n}\,
d|\nu|(z)\right|^\frac{2\theta}{n}\\
{}&{}&=\;\left|\int \sum^n_{r=0} (-1)^r\left(\begin{array}{c} n\\
r\end{array}\right)\Vert
k(\cdot,z)\Vert^{n-r}_{\mathcal{H}_k}\Vert k(\cdot,z_0)\Vert^r_{\mathcal{H}_k}\,
d|\nu|(z)\right|^\frac{2\theta}{n}\\
{}&{}&=\;\Bigg|\underbrace{\int k^{\frac{n}{2}}(z,z)\,d|\nu|(z)}_{A}\\
{}&{}&\quad+\;\underbrace{\sum^{n}_{r=1} (-1)^r\left(\begin{array}{c} n\\
r\end{array}\right) k^{\frac{r}{2}}(z_0,z_0)\int
k^{\frac{n-r}{2}}(z,z)\,
d|\nu|(z)}_{B}\Bigg|^\frac{2\theta}{n}.
\end{eqnarray*}
\normalsize
Note that the terms in $B$ are finite as for
$\theta\ge\frac{n}{2}\ge\frac{n-1}{2}\ge\cdots\ge\frac{1}{2}$, we have
$\Cal{M}^\theta_\rho(\Cal{Z})\subset\Cal{M}^{(n-1)/2}_k(\Cal{Z}
)\subset\cdots\subset\Cal{M}^1_k(\Cal{Z})\subset\Cal{M}^{1/2}_k(\Cal{Z}
)$ and therefore $A$ is finite, which means $\nu\in
\Cal{M}^{n/2}_k(\Cal{Z})$, i.e., $\Cal{M}^\theta_\rho(\Cal{Z})\subset
\Cal{M}^{n/2}_k(\Cal{Z})$ for $\theta\ge\frac{n}{2}$.
The result shows that
$\Cal{M}^\theta_{\rho}(\Cal{Z})=\Cal{M}^\theta_{k}(\Cal{Z})$ for all
$\theta\in\{\frac{n}{2}:n\in\bb{N}\}$.
\end{proof}

The above Proposition gives a natural interpretation of conditions on probability measures in terms of moments w.r.t. $\rho$. Namely, the kernel embedding $\mu_{k}(P)$,
where kernel $k$ generates the semimetric $\rho$, exists for every
$P$ with finite half-moment w.r.t.~$\rho$, and thus, MMD between $P$ and $Q$, $\gamma_{k}(P,Q)$
is well defined whenever both $P$ and $Q$ have finite half-moments
w.r.t.~$\rho$. If, in addition, $P$ and $Q$ have finite first moments
w.r.t.~$\rho$, then the $\rho$-energy distance between
$P$ and $Q$ is also well defined and it must be equal to the MMD, by Theorem \ref{thm: 2sample_dkern}.

Rather than imposing the condition on Borel probability measures,
one may assume that the underlying semimetric space $(\mathcal{Z},\rho)$
of negative type is itself bounded, i.e., that
$\sup_{z,z'\in\mathcal{Z}}\rho(z,z')<\infty$,
implying that distance kernels are bounded functions, and that both MMD and energy distance
are always defined. Conversely, bounded kernels (such as Gaussian)
always induce bounded semimetrics.

\begin{center}
\begin{table*}[t!]
\caption{MMD with distance kernels on data from \citet{GreFukHarSri09_short}.
Dimensionality is: \emph{Neural I} (64), \emph{Neural II} (100), \emph{Health
status} (12,600), \emph{Subtype} (2,118). The boldface denotes instances
where distance kernel had smaller Type II error in comparison to Gaussian
kernel.\label{tab:MMD-results-1}}
\vspace{2mm}
\centering{}{\small }%
\begin{tabular}{|c|c|c|c|c|c|c|c|c|}
\hline 
 &  & {\small Gauss} & {\small dist ($1/3$)} & {\small dist ($2/3$)} & {\small
dist ($1$)} & {\small dist ($4/3$)} & {\small dist ($5/3$)} & {\small dist
($2$)}\tabularnewline
\hline 
\hline 
\emph{\small Neural I}{\small{} } & {\small 1- Type I} & {\small .956} & {\small .969} & {\small .964} & {\small .949} & {\small .952} & {\small .959} & {\small .959}\tabularnewline
\hline 
{\small ($m=200$)} & {\small Type II} & {\small .118} & {\small .170} & {\small .139} & {\small .119} & \textbf{\small .109} & \textbf{\small .089} & \textbf{\small .117}\tabularnewline
\hline 
{\small Neural I } & {\small 1- Type I} & {\small .950} & {\small .969} & {\small .946} & {\small .962} & {\small .947} & {\small .930} & {\small .953}\tabularnewline
\hline 
{\small ($m=250$)} & {\small Type II} & {\small .063} & {\small .075} & \textbf{\small .045} & \textbf{\small .041} & \textbf{\small .040} & {\small .065} & \textbf{\small .052}\tabularnewline
\hline 
\emph{\small Neural II} & {\small 1- Type I} & {\small .956} & {\small .968} & {\small .965} & {\small .963} & {\small .956} & {\small .958} & {\small .943}\tabularnewline
\hline 
{\small ($m=200$)} & {\small Type II} & {\small .292} & {\small .485} & {\small .346} & {\small .319} & {\small .297} & \textbf{\small .280} & \textbf{\small .290}\tabularnewline
\hline 
\emph{\small Neural II} & {\small 1- Type I} & {\small .963} & {\small .980} & {\small .968} & {\small .950} & {\small .952} & {\small .960} & {\small .941}\tabularnewline
\hline 
{\small ($m=250$)} & {\small Type II} & {\small .195} & {\small .323} & {\small .197} & \textbf{\small .189} & \textbf{\small .194} & \textbf{\small .169} & \textbf{\small .183}\tabularnewline
\hline 
\emph{\small Subtype } & {\small 1- Type I} & {\small .975} & {\small .974} & {\small .977} & {\small .971} & {\small .966} & {\small .962} & {\small .966}\tabularnewline
\hline 
{\small ($m=10$)} & {\small Type II} & {\small .055} & {\small .828} & {\small .237} & {\small .092} & \textbf{\small .042} & \textbf{\small .033} & \textbf{\small .024}\tabularnewline
\hline 
\emph{\small Health st.} & {\small 1- Type I} & {\small .958} & {\small .980} & {\small .953} & {\small .940} & {\small .954} & {\small .954} & {\small .955}\tabularnewline
\hline 
{\small ($m=20$)} & {\small Type II} & {\small .036} & {\small .037} & {\small .039} & {\small .081} & {\small .114} & {\small .120} & {\small .165}\tabularnewline
\hline 
\end{tabular}
\end{table*}
\end{center}

\subsection{Distance Correlation}\label{subsec:distancecorr}

The notion of distance covariance extends naturally to that of \emph{distance
variance} $\mathcal{V}^{2}(X)=\mathcal{V}^{2}(X,X)$ and that of \emph{distance
correlation} (in analogy to the Pearson product-moment correlation
coefficient):
\setlength{\arraycolsep}{0.0em}
\begin{eqnarray*}
\mathcal{R}^{2}(X,Y) &{} = {}& \begin{cases}
\frac{\mathcal{V}^{2}(X,Y)}{\mathcal{V}(X)\mathcal{V}(Y)}, &
\mathcal{V}(X)\mathcal{V}(Y)>0,\\
0, & \mathcal{V}(X)\mathcal{V}(Y)=0.
\end{cases}
\end{eqnarray*}
Distance correlation also has a straightforward interpretation in
terms of kernels as:
\setlength{\arraycolsep}{0.0em}
\begin{eqnarray*}
\mathcal{R}^{2}(X,Y) &{} = {}&
\frac{\mathcal{V}^{2}(X,Y)}{\mathcal{V}(X)\mathcal{V}(Y)}\\
 &{} ={} &
\frac{\gamma_{k}^{2}(P_{XY},P_{X}P_{Y})}{\gamma_{k}(P_{XX},P_{X}P_{X})\gamma_{k}
(P_{YY},P_{Y}P_{Y})}\\
 & {}={} & \frac{\left\Vert \Sigma_{XY}\right\Vert _{HS}^{2}}{\left\Vert
\Sigma_{XX}\right\Vert _{HS}\left\Vert \Sigma_{YY}\right\Vert _{HS}},
\end{eqnarray*}
where covariance operator
$\Sigma_{XY}:\mathcal{H}_{k_{\mathcal{X}}}\to\mathcal{H}_{k_{\mathcal{Y}}}$
is a linear operator for which $\left\langle \Sigma_{XY}f,g\right\rangle
_{\mathcal{H}_{k_{\mathcal{Y}}}}=\mathbb{E}_{XY}\left[f(X)g(Y)\right]-\mathbb{E}
_{X}f(X)\mathbb{E}_{Y}g(Y)$,
for all $f\in\mathcal{H}_{k_{\mathcal{X}}}$ and
$g\in\mathcal{H}_{k_{\mathcal{Y}}}$,
and $\left\Vert \cdot\right\Vert _{HS}$ denotes the Hilbert-Schmidt
norm \citep{GreHerSmoBouetal05}. It is clear that $\mathcal{R}$
is invariant to scaling $(X,Y)\mapsto(\epsilon X,\epsilon Y)$, $\epsilon>0$,
whenever the corresponding semimetrics are homogeneous, i.e., whenever
$\rho_{\mathcal{X}}(\epsilon x,\epsilon x')=\epsilon\rho_{\mathcal{X}}(x,x')$,
and similarly for $\rho_{\mathcal{Y}}$. Moreover, $\mathcal{R}$
is invariant to translations $(X,Y)\mapsto(X+x',Y+y')$, $x'\in\mathcal{X}$,
$y'\in\mathcal{Y}$, whenever $\rho_{\mathcal{X}}$ and $\rho_{\mathcal{Y}}$
are translation invariant.

\subsection{Further Experiments}\label{subsec:further}

We assessed performance of two-sample tests based on distance kernels
with various exponents and compared it to that of a Gaussian kernel
on real-world multivariate datasets: \emph{Health st.} (microarray
data from normal and tumor tissues), \emph{Subtype} (microarray data
from different subtypes of cancer) and \emph{Neural I/II} (local field
potential (LFP) electrode recordings from the Macaque primary visual
cortex (V1) with and without spike events), all discussed in
\citet{GreFukHarSri09_short}.
In contrast to \citet{GreFukHarSri09_short}, we used smaller sample
sizes, so that some Type II error persists. At higher sample sizes,
all tests exhibit Type II error which is virtually zero. The results
are reported in Table \ref{tab:MMD-results-1} below. We used the
spectral test for all experiments, and the reported averages are obtained
by running 1000 trials. We note that for dataset \emph{Subtype} which
is high dimensional but with only a small number of dimensions varying
in mean, a larger exponent results in a test of greater power.

\end{document}